\documentclass{article}

\usepackage[final]{neurips_2025}

\usepackage[utf8]{inputenc} 
\usepackage[T1]{fontenc}    
\usepackage{hyperref}       
\usepackage{url}            
\usepackage{booktabs}       
\usepackage{amsfonts}       
\usepackage{nicefrac}       
\usepackage{microtype}      
\usepackage{xcolor}         

\usepackage{caption}
\usepackage{subcaption}
\usepackage{wrapfig}
\usepackage{etoolbox}
\usepackage{verbatim}
\usepackage[inline]{enumitem}
\usepackage{amsmath}
\usepackage{graphicx}
\usepackage{tikz}

\usepackage[most]{tcolorbox}

\newcommand{\algoOurs}{\textsc{UserAlign}}
\newcommand{\algoOursLoss}{\ensuremath{\algoOurs{}_{\textsc{loss}}}}

\usepackage[usenames,dvipsnames,table]{xcolor}
\usepackage{hyperref}
\hypersetup{ %
	pdfauthor={},
	pdftitle={},
	pdfsubject={},
    pdflang={en},
	bookmarks,
	bookmarksnumbered,
	backref=page,
	pageanchor=true,
	plainpages=false,
	colorlinks=true,%
	linktocpage=true,
	citecolor=MidnightBlue,%
	filecolor=BrickRed,%
	linkcolor=NavyBlue,%
	urlcolor=NavyBlue 
}

\usepackage{multirow}
\usepackage{colortbl}
\usepackage{arydshln}

\usepackage{algorithm}
\usepackage{algcompatible}
\usepackage[noend]{algpseudocode}

\usepackage{amsthm}

\newtheorem{theorem}{Theorem}

\newtheorem{lemma}{Lemma}
\newtheorem{proposition}{Proposition}

\DeclareMathOperator*{\argmin}{arg\,min}
\DeclareMathOperator*{\argmax}{arg\,max}

\newcommand\norm[1]{\left\lVert#1\right\rVert}

\newcommand{\bcc}[1]{\left\{{#1}\right\}}
\newcommand{\brr}[1]{\left({#1}\right)}
\newcommand{\bss}[1]{\left[{#1}\right]}

\newcommand{\abs}[1]{\left\vert#1\right\vert}

\newcommand{\Prob}[1]{\mathbb{P}\bss{{#1}}}

\newcommand{\LINECOMMENT}[1]{\STATE {\color{brown}\ttfamily\small \(\triangleright\) #1}}

\newcommand*\circled[1]{\tikz[baseline=(char.base)]{
            \node[shape=circle,draw=gray,inner sep=1.2pt, font=\footnotesize\bfseries, fill=gray, text=white] (char) {#1};}
    }
\allowdisplaybreaks

\newtoggle{MainSuppContent}
\settoggle{MainSuppContent}{true}

\newtoggle{MainContentOnly}
\settoggle{MainContentOnly}{false}

\newtoggle{SuppContentOnly}
\settoggle{SuppContentOnly}{false}

{\title{Inference-Time Personalized Alignment \\ with a Few User Preference Queries}}

\iftoggle{SuppContentOnly}{
    \title{Inference-Time Personalized Alignment \\ with a Few User Preference Queries \\ \textnormal{Supplementary Material}}
}

\author{%
  Victor-Alexandru P{\u a}durean\ \\
  MPI-SWS\\
  \texttt{vpadurea@mpi-sws.org} \\
  \And
  Parameswaran Kamalaruban \\
  Featurespace Innovation Lab, Visa \\ 
  \texttt{kaparame@visa.com} \\
  \And
  Nachiket Kotalwar \\
  Carnegie Mellon University \\
  \texttt{nkotalwa@cs.cmu.edu} \\
  \And
  Alkis Gotovos \\
  MPI-SWS \\
  \texttt{agkotovo@mpi-sws.org} \\
  \And
  Adish Singla \\
  MPI-SWS \\
  \texttt{adishs@mpi-sws.org} \\
}

\begin{document}

\maketitle

\iftoggle{MainSuppContent}{
\begin{abstract}
We study the problem of aligning a generative model's response with a user's preferences. Recent works have proposed several different formulations for personalized alignment; however, they either require a large amount of user preference queries or require that the preference be explicitly specified as a text input. In this paper, we propose a novel inference-time personalized alignment method, \algoOurs{}, that elicits the user's preferences with a few queries as pairwise response comparisons. In particular, \algoOurs{} builds on the theoretical framework of best-arm identification in logistic bandits and selects a personalized response from a fixed pool of the model's generated responses. The key idea is to consider the user's feedback consistent and noise-free, and incorporate it into the theoretical framework to identify the best response quickly. Experimental results across several tasks,  involving personalized text and image generation, showcase the effectiveness of \algoOurs{} in achieving personalized alignment.
\end{abstract}

%
\section{Introduction}\label{sec:intro}

Generative models have demonstrated remarkable capabilities across language and vision tasks, yet aligning their outputs with human preferences remains a central challenge~\cite{DBLP:journals/corr/abs-2307-12966,DBLP:journals/tmlr/CasperDSGSRFKLF23}. While population-level alignment methods such as Reinforcement Learning from Human Feedback (RLHF)~\cite{DBLP:conf/nips/ChristianoLBMLA17,DBLP:journals/corr/abs-1909-08593,DBLP:conf/nips/Ouyang0JAWMZASR22} and Direct Preference Optimization (DPO)~\cite{DBLP:conf/nips/RafailovSMMEF23,DBLP:journals/corr/abs-2402-01306,DBLP:conf/nips/0001X024} have made significant strides, practical applications often demand personalization. Users exhibit highly individual tastes and requirements, from stylistic writing preferences to visual aesthetics to lifestyle preferences, which generic alignment cannot fully capture. Consequently, our key research question is: \emph{How can we align a generative model’s response to a specific user on the fly, where the user’s preferences need to be elicited with limited interaction?}

Recent efforts toward personalized alignment have explored both training-time and inference-time strategies, each with drawbacks when query budgets are small. Training-time personalization approaches fine-tune models on user-specific data but typically rely on extensive preference annotations~\cite{DBLP:journals/corr/abs-2405-00254,DBLP:journals/corr/abs-2402-05133}. Inference-time methods offer more flexibility by adapting model outputs at deployment; however, many require users to articulate their preferences as explicit text prompts~\cite{DBLP:conf/iclr/LinRLDSCB024,DBLP:journals/corr/abs-2402-06147}, which can be cognitively demanding and imprecise for complex tastes. Even theoretically grounded active learning and bandit-based methods for modeling latent reward functions demand a lot of pairwise comparisons to converge to a reliable preference estimate~\cite{DBLP:conf/icml/YueJ09,DBLP:conf/rss/SadighDSS17}, making them impractical for real-world, low-interaction settings.

%
\tcbset{
    question/.style={
        colback=gray!2,
        colframe=gray,
        top=0.2mm,
        bottom=0.2mm,
        left=0mm,
        right=0mm
    }
}

\begin{figure*}
    \centering
    \begin{subfigure}{1\textwidth}
        \input{figs/1_introduction/example_2/example_2_question}
        \label{fig:user-query}
    \end{subfigure}

    \vspace{-1mm}

    \begin{subfigure}{1\textwidth}
        \input{figs/1_introduction/example_2/example_2_generations}
        \label{fig:candidates}
    \end{subfigure}

    \vspace{-1mm}

    \begin{minipage}{0.78\textwidth}
        \begin{subfigure}{\textwidth}
            \noindent
\begin{tcolorbox}[colback=gray!2, colframe=gray, top=0.2mm, bottom=0.2mm, left=0.2mm, right=0.2mm]
\hspace{-1mm} \circled{3} \algoOurs{} eliciting user preferences via pairwise comparisons
\begin{center}
    \setlength{\fboxsep}{0pt}
    \setlength{\fboxrule}{2.5pt} 

    \begin{minipage}{0.30\textwidth}
    \centering
    \begin{tabular}{@{}c@{\hspace{0.5em}}c@{}}
        \parbox[c]{1em}{\centering\textcolor{black}{\small \textit{t=}$1$}} & 
        \parbox[c]{\dimexpr0.9\linewidth}{%
        \includegraphics[width=0.38\textwidth]{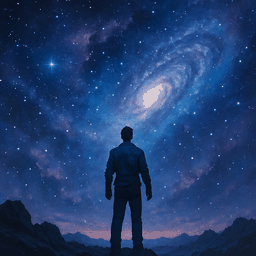}%
        \hspace{0.02\textwidth}%
        \fcolorbox{orange!90!white}{white}{\includegraphics[width=0.38\textwidth]{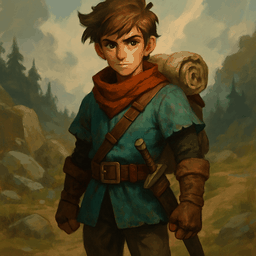}}%
        }
    \end{tabular}
    \end{minipage}
    \hfill
    \begin{minipage}{0.30\textwidth}
    \centering
    \begin{tabular}{@{}c@{\hspace{0.5em}}c@{}}
        \parbox[c]{1em}{\centering\textcolor{black}{\small \textit{t=}$2$}} & 
        \parbox[c]{\dimexpr0.9\linewidth}{%
        \fcolorbox{orange!90!white}{white}{\includegraphics[width=0.38\textwidth]{figs/1_introduction/example_2/images/none_13_gameimage512d_Q1.png}}%
        \hspace{0.02\textwidth}%
        \includegraphics[width=0.38\textwidth]{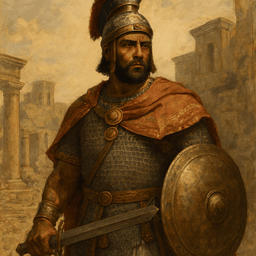}%
        }
    \end{tabular}
    \end{minipage}
    \hfill
    \begin{minipage}{0.30\textwidth}
    \centering
    \begin{tabular}{@{}c@{\hspace{0.5em}}c@{}}
        \parbox[c]{1em}{\centering\textcolor{black}{\small \textit{t=}$3$}} & 
        \parbox[c]{\dimexpr0.9\linewidth}{%
        \fcolorbox{orange!90!white}{white}{\includegraphics[width=0.38\textwidth]{figs/1_introduction/example_2/images/none_13_gameimage512d_Q1.png}}%
        \hspace{0.02\textwidth}%
        \includegraphics[width=0.38\textwidth]{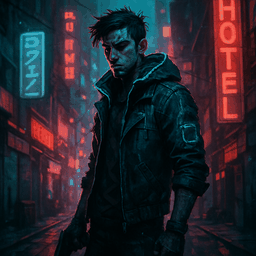}%
        }
    \end{tabular}
    \end{minipage}
    
    \vspace{2mm}
    \begin{minipage}{0.30\textwidth}
    \centering
    \begin{tabular}{@{}c@{\hspace{0.5em}}c@{}}
        \parbox[c]{1em}{\centering\textcolor{black}{\small \textit{t=}$4$}} & 
        \parbox[c]{\dimexpr0.9\linewidth}{%
        \fcolorbox{orange!90!white}{white}{\includegraphics[width=0.38\textwidth]{figs/1_introduction/example_2/images/none_13_gameimage512d_Q1.png}}%
        \hspace{0.02\textwidth}%
        \includegraphics[width=0.38\textwidth]{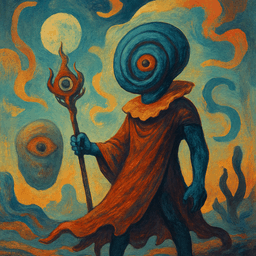}%
        }
    \end{tabular}
    \end{minipage}
    \hfill
    \begin{minipage}{0.30\textwidth}
    \centering
    \begin{tabular}{@{}c@{\hspace{0.5em}}c@{}}
        \parbox[c]{1em}{\centering\textcolor{black}{\small \textit{t=}$5$}} & 
        \parbox[c]{\dimexpr0.9\linewidth}{%
        \includegraphics[width=0.38\textwidth]{figs/1_introduction/example_2/images/none_13_gameimage512d_Q1.png}%
        \hspace{0.02\textwidth}%
        \fcolorbox{orange!90!white}{white}{\includegraphics[width=0.38\textwidth]{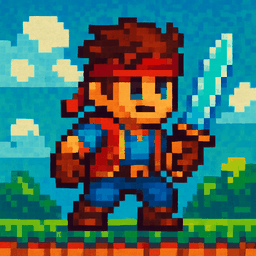}}%
        }
    \end{tabular}
    \end{minipage}
    \hfill
    \begin{minipage}{0.30\textwidth}
    \centering
    \begin{tabular}{@{}c@{\hspace{0.5em}}c@{}}
        \parbox[c]{1em}{\centering\textcolor{black}{\small \textit{t=}$6$}} & 
        \parbox[c]{\dimexpr0.9\linewidth}{%
        \fcolorbox{orange!90!white}{white}{\includegraphics[width=0.38\textwidth]{figs/1_introduction/example_2/images/RetroGamer_1_gameimage512d_Q1.png}}%
        \hspace{0.02\textwidth}%
        \includegraphics[width=0.38\textwidth]{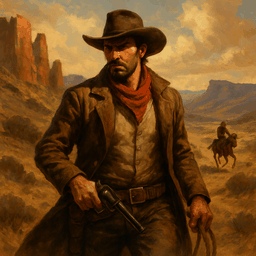}%
        }
    \end{tabular}
    \end{minipage}
    
    \vspace{2mm}
    \begin{minipage}{0.30\textwidth}
    \centering
    \begin{tabular}{@{}c@{\hspace{0.5em}}c@{}}
        \parbox[c]{1em}{\centering\textcolor{black}{\small \textit{t=}$7$}} & 
        \parbox[c]{\dimexpr0.9\linewidth}{%
        \fcolorbox{orange!90!white}{white}{\includegraphics[width=0.38\textwidth]{figs/1_introduction/example_2/images/RetroGamer_1_gameimage512d_Q1.png}}%
        \hspace{0.02\textwidth}%
        \includegraphics[width=0.38\textwidth]{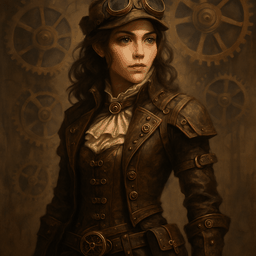}%
        }
    \end{tabular}
    \end{minipage}
    \hfill
    \begin{minipage}{0.30\textwidth}
    \centering
    \begin{tabular}{@{}c@{\hspace{0.5em}}c@{}}
        \parbox[c]{1em}{\centering\textcolor{black}{\small \textit{t=}$8$}} & 
        \parbox[c]{\dimexpr0.9\linewidth}{%
        \fcolorbox{orange!90!white}{white}{\includegraphics[width=0.38\textwidth]{figs/1_introduction/example_2/images/RetroGamer_1_gameimage512d_Q1.png}}%
        \hspace{0.02\textwidth}%
        \includegraphics[width=0.38\textwidth]{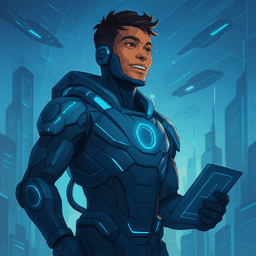}%
        }
    \end{tabular}
    \end{minipage}
    \hfill
    \begin{minipage}{0.30\textwidth}
    \centering
    \begin{tabular}{@{}c@{\hspace{0.5em}}c@{}}
        \parbox[c]{1em}{\centering\textcolor{black}{\small \textit{t=}$9$}} & 
        \parbox[c]{\dimexpr0.9\linewidth}{%
        \fcolorbox{orange!90!white}{white}{\includegraphics[width=0.38\textwidth]{figs/1_introduction/example_2/images/RetroGamer_1_gameimage512d_Q1.png}}%
        \hspace{0.02\textwidth}%
        \includegraphics[width=0.38\textwidth]{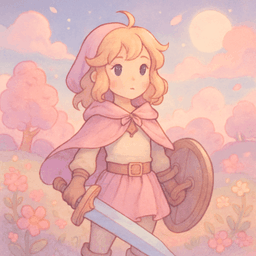}%
        }
    \end{tabular}
    \end{minipage}

\end{center}
\end{tcolorbox}
            \label{fig:refinement}
        \end{subfigure}
    \end{minipage}%
    \hfill
    \begin{minipage}{0.21\textwidth}
    \vspace{-2mm}
     \begin{tcolorbox}[colback=gray!2, colframe=gray, top=0.2mm, bottom=11mm, left=0mm, right=0mm]
            \circled{4} Output
            \\
            \\
            \\    
            \\
            \includegraphics[width=1\textwidth]{figs/1_introduction/example_2/images/RetroGamer_1_gameimage512d_Q1.png} 
          
          \end{tcolorbox}
    \end{minipage}

    \caption{\looseness-1An illustrative example showcasing inference-time personalized alignment methodology. Starting with a user question (Stage 1), the system generates a pool of responses (Stage 2). Then, \algoOurs{} iteratively collects user preferences (Stage 3) via pairwise comparisons to determine the most suitable response---user's preferred responses are {\color{orange}{highlighted}}. At the end, the final response is selected (Stage 4). The user in this example is simulated by the GPT-4o-mini model conditioned on the persona description ``\emph{NostalgicExplorer: A 36-year-old who grew up with classic platformers and adventure games. Loves timeless heroes with a sense of wonder and a hint of retro charm''}.
    }
    \vspace{-4.8mm}
    \label{fig.introduction}
\end{figure*}

In this paper, we introduce \algoOurs{}, a novel inference-time personalized alignment method that efficiently elicits user preferences through only a few pairwise comparisons among a fixed pool of candidate responses. Building on the best-arm identification framework in logistic bandits~\cite{DBLP:conf/nips/FilippiCGS10,DBLP:conf/nips/Abbasi-YadkoriPS11,DBLP:conf/aistats/AbeilleFC21,xu2018fully,lee2024improved}, \algoOurs{} maintains a loss-based confidence region over the user’s latent preference model and aggressively shrinks this region by treating each comparison as consistent and noise-free. By leveraging version-space elimination via intersecting halfspaces defined by the observed duels~\cite{DBLP:journals/ml/Littlestone87,DBLP:books/daglib/0033642}, our method rapidly isolates the best response without extensive querying.
Figure~\ref{fig.introduction} provides an illustrative example showcasing \algoOurs{}'s interaction with the user on an image generation task.
Our main results and contributions are summarized below:
\begin{itemize}[leftmargin=8pt]
\item We formulate the problem of inference-time personalized alignment with a particular focus on practical settings where the user's preferences need to be elicited with limited interaction. (Section~\ref{sec.problemsetup})
\item We develop a novel method, \algoOurs{}, theoretically grounded in the logistic bandits framework, and achieve fast alignment by modeling the user's feedback as consistent and noise-free. (Section~\ref{sec.method})
\item We demonstrate that \algoOurs{} can achieve fast personalized alignment with a few preference queries in personalized text and image generation, evaluated on both simulated and real users. We release our implementation and datasets to support further research.\footnote{Github repo: \href{https://github.com/machine-teaching-group/neurips2025-useralign}{https://github.com/machine-teaching-group/neurips2025-useralign}.} (Sections~\ref{sec.experiments}~and~\ref{sec.userstudy})
\end{itemize}

%
%


\section{Related Work}\label{sec:related-work}
\vspace{-3mm}
%
\begin{table*}[t!]
    \centering
    \caption{Related work on preference alignment of generative models; see Section~\ref{sec:related-work} for details.}
    \label{fig.related_work}
    \begin{subfigure}{1\textwidth}
        \centering
        \caption{\textbf{Preference alignment methods.} General preference alignment methods use aggregated population-level preference data, while personalized methods rely on user-level data (with user identity).}
        \scalebox{0.75}{
        \setlength\tabcolsep{12pt}
        \renewcommand{\arraystretch}{1.1}
            \begin{tabular}{llll}
                \toprule
                & & \textbf{Training-time} & \textbf{Inference-time} \\
                \midrule
                \multirow{4}{*}{\textbf{General Preferences}} 
                & \textbf{Offline} & RLHF~\cite{DBLP:conf/nips/ChristianoLBMLA17,DBLP:journals/corr/abs-1909-08593,DBLP:conf/nips/Ouyang0JAWMZASR22} & BoN~\cite{DBLP:journals/corr/abs-2009-01325,DBLP:journals/corr/abs-2112-09332} \\
                &        & DPO~\cite{DBLP:conf/nips/RafailovSMMEF23,DBLP:journals/corr/abs-2402-01306,DBLP:conf/nips/0001X024} & Reward guided decoding~\cite{DBLP:conf/emnlp/DengR23,DBLP:conf/iclr/KhanovBL24,DBLP:journals/corr/abs-2410-16033} \\
                \cmidrule{2-4}
                & \textbf{Online} & Active RLHF~\cite{das2024active,DBLP:journals/corr/abs-2402-09401} & Active BoN~\cite{DBLP:journals/corr/abs-2402-13210,DBLP:journals/corr/abs-2411-01493} \\
                &        & Active DPO~\cite{DBLP:journals/corr/abs-2402-09401,DBLP:journals/corr/abs-2411-01493,DBLP:conf/icml/MuldrewHZB24} & Active Bayesian PM~\cite{DBLP:conf/rss/SadighDSS17,DBLP:conf/icml/DwaracherlaAHR24,DBLP:conf/nips/MeloTAG24} \\
                \midrule
                \multirow{8}{*}{\textbf{Personal Preferences}} 
                & \textbf{Offline} & Personalized RLHF~\cite{DBLP:journals/corr/abs-2405-00254} & URIAL~\cite{DBLP:conf/iclr/LinRLDSCB024}, DeAL~\cite{DBLP:journals/corr/abs-2402-06147} \\
                &        & Personalized DPO~\cite{DBLP:journals/corr/abs-2402-05133} & PAD~\cite{chen2025pad}, OPAD~\cite{zhu2025onthefly}, Amulet~\cite{zhang2025amulet} \\
                &        & VPL~\cite{DBLP:conf/nips/PoddarWIGJ24} & Personalized Soups~\cite{DBLP:journals/corr/abs-2310-11564}, MOD~\cite{DBLP:conf/nips/ShiCHLHSD24} \\
                &        & & LoRE~\cite{bose2025lore}, PAL~\cite{chen2025pal} \\
                \cmidrule{2-4}
                & \textbf{Online} &  & PBO~\cite{DBLP:conf/icml/GonzalezDDL17}, APL~\cite{DBLP:journals/corr/abs-2411-00524} \\
                &        &  & Active BoN~\cite{DBLP:journals/corr/abs-2402-13210,DBLP:journals/corr/abs-2411-01493} \\
                &        &  & Active Bayesian pref. model~\cite{DBLP:conf/rss/SadighDSS17,DBLP:conf/icml/DwaracherlaAHR24,DBLP:conf/nips/MeloTAG24} \\
                &        &  & \colorbox{green!20}{\algoOurs{}} \\
                \bottomrule
            \end{tabular}
        }
        \label{fig.related_work.table1}        
    \end{subfigure}
    \\
    \vspace{3mm}
    \begin{subfigure}{1\textwidth}
        \centering
        \caption{\looseness-1\textbf{Inference-time personalized preference alignment methods.} User preference input can take the form of explicit text (e.g., preference specifications or prompt-response examples), pairwise comparisons of base model outputs, or weights over predefined objectives. Offline methods impose less user load than online methods, as they do not require active user interaction. Warmup options include training a personalized preference model with multi-user pairwise data, or training an ensemble of generative models for different objectives. Inference-time operations include in-context learning, active preference learning from user comparisons, logit adjustment using on-the-fly reward functions with text-based preference input, logit adjustment via learned user preference weights with a pre-trained preference model, or logit adjustment using an ensemble of pre-trained generative models and user-defined weights. Online methods typically account for uncertainty in user preference modeling. Aligned responses can be generated via guided decoding or selection from a pool of pre-generated outputs.}        
        \scalebox{0.75}{
        \setlength\tabcolsep{4pt}
        \renewcommand{\arraystretch}{0.95}
            \begin{tabular}{lllllll}
                \toprule
                & \textbf{\begin{tabular}[c]{@{}l@{}}Preference\\ Input\end{tabular}} & \textbf{\begin{tabular}[c]{@{}l@{}}User\\ Load\end{tabular}} & \textbf{\begin{tabular}[c]{@{}l@{}}Warmup \\ Operation\end{tabular}} & \textbf{\begin{tabular}[c]{@{}l@{}}Test-Time\\ Operation\end{tabular}} & \textbf{\begin{tabular}[c]{@{}l@{}}Uncertainty \\ Quantification\end{tabular}} & \textbf{\begin{tabular}[c]{@{}l@{}}Response\\ Generation\end{tabular}} \\
                \midrule
                URIAL~\cite{DBLP:conf/iclr/LinRLDSCB024} & Text & Low & None & In-context learn  & No & Guiding \\
                DeAL~\cite{DBLP:journals/corr/abs-2402-06147} & Text & Low  & None & In-context learn & No & Guiding \\
                \midrule
                PAD~\cite{chen2025pad} & Text & Low & Train pref. model & On-the-fly pref. & No & Guiding \\
                OPAD~\cite{zhu2025onthefly} & Text & Low & None & On-the-fly pref. & No & Guiding \\
                Amulet~\cite{zhang2025amulet} & Text & Low & None & On-the-fly pref. & No & Guiding \\
                \midrule
                Personalized Soups~\cite{DBLP:journals/corr/abs-2310-11564} & Weight & Low & Train gen. models & Ensemble & No & Guiding \\
                MOD~\cite{DBLP:conf/nips/ShiCHLHSD24} & Weight & Low & Train gen. models & Ensemble & No & Guiding \\
                \midrule
                LoRE~\cite{bose2025lore} & Comparisons & Low & Train pref. model & Learn pref. weight & No & Guiding \\   
                PAL~\cite{chen2025pal} & Comparisons & Low & Train pref. model & Learn pref. weight & No & Guiding \\
                \midrule
                PBO~\cite{DBLP:conf/icml/GonzalezDDL17} & Comparisons & High & None & Active pref. model & Yes & Selection \\
                APL~\cite{DBLP:journals/corr/abs-2411-00524} & Comparisons & High & None & Active pref. model & Yes & Selection \\
                Active BoN~\cite{DBLP:journals/corr/abs-2402-13210,DBLP:journals/corr/abs-2411-01493} & Comparisons & High & None & Active pref. model & Yes & Selection \\
                Active Bayesian pref. model~\cite{DBLP:conf/rss/SadighDSS17} & Comparisons & High & None & Active pref. model & Yes & Selection \\
                \rowcolor{green!20}
                \algoOurs{} & Comparisons & Low & None & Active pref. model & Yes & Selection \\
                %
                \bottomrule
            \end{tabular}
        }
        \label{fig.related_work.table2}        
    \end{subfigure}
    \vspace{-2mm}
\end{table*}

\looseness-1A broad range of general preference alignment methods has emerged (see Table~\ref{fig.related_work.table1}). Offline approaches like Reinforcement Learning from Human Feedback (RLHF)~\cite{DBLP:conf/nips/ChristianoLBMLA17,DBLP:journals/corr/abs-1909-08593,DBLP:conf/nips/Ouyang0JAWMZASR22} and Direct Preference Optimization (DPO)~\cite{DBLP:conf/nips/RafailovSMMEF23,DBLP:journals/corr/abs-2402-01306,DBLP:conf/nips/0001X024} fine-tune models using aggregated preference data. At inference, methods like Best-of-N sampling (BoN)~\cite{DBLP:journals/corr/abs-2009-01325,DBLP:journals/corr/abs-2112-09332} and reward-guided decoding~\cite{DBLP:conf/emnlp/DengR23,DBLP:conf/iclr/KhanovBL24,DBLP:journals/corr/abs-2410-16033} steer outputs without retraining. To lower annotation cost, online variants such as Active RLHF~\cite{das2024active,DBLP:journals/corr/abs-2402-09401} and Active DPO~\cite{DBLP:journals/corr/abs-2402-09401,DBLP:journals/corr/abs-2411-01493,DBLP:conf/icml/MuldrewHZB24} adapt losses on-the-fly, while methods like Active BoN~\cite{DBLP:journals/corr/abs-2402-13210,DBLP:journals/corr/abs-2411-01493} and Active Bayesian preference modeling~\cite{DBLP:conf/rss/SadighDSS17,DBLP:conf/icml/DwaracherlaAHR24,DBLP:conf/nips/MeloTAG24} refine decoding via sequential pairwise feedback. These methods capture broad community norms but struggle to personalize outputs for individual users.

In contrast, personalized preference alignment methods leverage user-specific data to generate individualized outputs (see Table~\ref{fig.related_work.table2}). Offline methods~\cite{DBLP:journals/corr/abs-2405-00254,DBLP:journals/corr/abs-2402-05133,DBLP:conf/nips/PoddarWIGJ24} fine-tune models on per-user preference data. At inference, text-based logit adjustment methods like URIAL~\cite{DBLP:conf/iclr/LinRLDSCB024} and DeAL~\cite{DBLP:journals/corr/abs-2402-06147} embed few-shot examples, while PAD~\cite{chen2025pad}, OPAD~\cite{zhu2025onthefly}, and Amulet~\cite{zhang2025amulet} learn lightweight reward functions on-the-fly. Personalization is also supported by weight-based ensembling (Personalized Soups~\cite{DBLP:journals/corr/abs-2310-11564}, MOD~\cite{DBLP:conf/nips/ShiCHLHSD24}) and comparison-driven logit adjustment (LoRE~\cite{bose2025lore}, PAL~\cite{chen2025pal}). Online personalization algorithms, such as PBO~\cite{DBLP:conf/icml/GonzalezDDL17} and APL~\cite{DBLP:journals/corr/abs-2411-00524}, actively query users for pairwise comparisons to refine decoders in real time; Active BoN and Active Bayesian preference modeling naturally extend to this user-specific setting. Recent surveys provide broader overviews of personalized alignment~\cite{xie2025survey,DBLP:journals/corr/abs-2503-17003}.

\looseness-1Another line of related work is Bayesian active learning that treats a user’s latent reward as a random variable and selects queries that most reduce posterior uncertainty~\cite{DBLP:conf/rss/SadighDSS17,DBLP:conf/corl/BiyikS18}, enabling sample-efficient recovery of preference weights under mild assumptions. Similarly, the dueling bandit framework models preference learning as an online decision problem with pairwise feedback~\cite{DBLP:conf/icml/YueJ09,DBLP:conf/colt/YueBKJ09}, where algorithms using upper-confidence bounds or Thompson sampling achieve sublinear regret and convergence guarantees~\cite{DBLP:conf/icml/ZoghiWMR14,DBLP:conf/icml/SahaG22}. However, both approaches often require many user queries in practice.

\vspace{-1mm}
\section{Problem Formulation}\label{sec.problemsetup}
\vspace{-2mm}

\textbf{System-user interaction.} The envisioned system is powered by a generative model and interacts with a user $u$. The generative model is a stochastic mapping $\pi: \mathcal{X} \to \Delta(\mathcal{Y})$, where $\mathcal{X}$ is the input space and $\mathcal{Y}$ is the output space ($\Delta(\mathcal{Y})$ denotes the probability simplex over $\mathcal{Y}$). The user’s preferences are captured by an unknown latent preference model: for any input $x \in \mathcal{X}$ and a pair of responses $(y, y') \in \mathcal{Y}^2$, the probability that the user prefers $y$ over $y'$ (denoted $y \succ y'$) is given by the preference model $p_u [y \succ y' \mid x]$. Given an input $x \in \mathcal{X}$ from the user, the system first generates a candidate pool of responses $\mathcal{Y}_\text{cand}$ using the generative model $\pi$, and then seeks to output a preference-aligned response $\widehat{y}$ from this pool. To this end, the system must learn or infer the user's preference model for $x$ through interaction. This interaction consists of querying the user to express a preference over a pair $(y, y') \in \mathcal{Y}_\text{cand}^2$ from the pool. Throughout, the system maintains and updates a dataset of preferences $\mathcal{D} = \{(x, y, y', r)\}$, where $r = 1$ if the user prefers $y$ over $y'$ for $x$, and $r = 0$ otherwise. The complete interaction process is described in Algorithm~\ref{alg:setup}.

\textbf{Objective.} Let $\widetilde{y} \in \mathcal{Y}$ be a baseline response for $x$ generated by the model $\pi$ with zero sampling temperature. This response is produced prior to any user interaction. For input $x$, we define the \emph{win-rate} of any response $y \in \mathcal{Y}$ against the baseline $\widetilde{y}$ as $p_u [y \succ \widetilde{y} \mid x]$. The system aims to output a response that approximately maximizes this win-rate while using as few preference queries as possible. Specifically, given the candidate pool $\mathcal{Y}_\text{cand}$ for $x$, the goal is to find the response
$y^\star = \argmax_{y \in \mathcal{Y}_\text{cand}} p_u [y \succ \widetilde{y} \mid x]$ with minimal interaction.

\begin{algorithm*}[t]
    \caption{System-User Interaction}
    \begin{algorithmic}[1]
        \State User $u$ provides an input prompt $x \in \mathcal{X}$ to the system. 
        \State System uses the generative model $\pi$ to generate a pool of responses $\mathcal{Y}_\text{cand}$ for $x$. 
        \While{\emph{stopping criteria not met}}
            \Statex \qquad System selects a pair of responses $(y, y') \in \mathcal{Y}_\text{cand}^2$. 
            \Statex \qquad System asks the user to provide a preference over the pair $(y, y')$.
            \Statex \qquad System sets $r = 1$ if the user prefers $y$ over $y'$, and $r = 0$ otherwise.
            \Statex \qquad System updates the preference dataset $\mathcal{D}$ with the preference tuple $(x, y, y', r)$.
        \EndWhile{}
        \State System selects a final response $\widehat{y} \in \mathcal{Y}_\text{cand}$ for the user.  
    \end{algorithmic}
    \label{alg:setup}
\end{algorithm*} 

%

\vspace{-1mm}
\section{Methodology}\label{sec.method}
\vspace{-2mm}

In this section, we present our algorithm, \algoOurs{}, for generating user preference-aligned responses using best-arm identification methods from logistic bandits literature. The full procedure is outlined in Algorithm~\ref{alg:general-pad}, with its subroutine \textsc{Solve} detailed in Algorithm~\ref{alg:solve-pad}.

\textbf{Preliminaries.} For any input $x \in \mathcal{X}$ and a pair of responses $(y, y') \in \mathcal{Y}^2$, the Bradley-Terry-Luce (BTL) preference model is defined as $p_\textnormal{BTL}\bss{y \succ y' ~\big|~ x, \theta} ~:=~ \mu (\langle {\theta}, {\phi(x,y) - \phi(x, y')} \rangle)$, where $\theta \in \Theta \subset \mathbb{R}^d$ is a weight vector, $\phi: \mathcal{X} \times \mathcal{Y} \to \mathbb{R}^d$ is a feature mapping, and $\mu(z) = 1 / (1 + e^{-z})$ denotes the logistic function. To develop our algorithm, we consider a user preference model $p_u$ following the BTL model with some unknown $\theta^\star \in \Theta$. We also adopt the following standard assumption~\cite{DBLP:conf/icml/FauryACF20}: $\norm{\phi(x, y) - \phi(x, y')}_2 \leq 1, \forall x \in \mathcal{X}, y, y' \in \mathcal{Y}$ and $\norm{\theta^\star}_2 \leq S$ with known $S > 0$. We define $\kappa^\star_{\mathcal{X}, \mathcal{Y}} := \max_{x \in \mathcal{X}} \max_{y, y' \in \mathcal{Y}} \frac{1}{\dot{\mu}(\langle \theta^\star, \phi(x, y) - \phi(x, y') \rangle)}$ and $\Delta_{\mathcal{Y}_\text{cand}} := \min_{y, y' \in \mathcal{Y}_\text{cand}; y \neq y'} \langle \theta^\star, \phi(x, y) - \phi(x, y') \rangle$. Under this setup, the response maximizing the win-rate can be equivalently written as: 
\begin{equation}
\label{eq:optimal-aligned-response}
y^\star ~=~ \argmax_{y \in \mathcal{Y}_\text{cand}} \langle \theta^\star, \phi(x, y) \rangle .
\end{equation}

\subsection{Theoretical Framework for \algoOurs{}}
\label{sec:theoretical-analysis}

\looseness-1Here, we describe \algoOurs{} without line~4 in the \textsc{Solve} subroutine, referring to this variant as \algoOursLoss{}. Starting with an empty preference dataset $\mathcal{D}_0 = \bcc{}$, we actively populate it with preference tuples $(x, y^{(1)}, y^{(2)}, r)$, where $r = 1$ if user $u$ prefers $y^{(1)}$ over $y^{(2)}$, and $r = 0$ otherwise. Let $\mathcal{D}_{t} = \{(x_\tau, y_\tau^{(1)}, y_\tau^{(2)}, r_\tau)\}_{\tau = 0}^{t-1}$ be the dataset at step $t$. We select a pair of responses $(y_t^{(1)}, y_t^{(2)})$ for which we request user feedback and then add the resulting tuple $(x, y_t^{(1)}, y_t^{(2)}, r_t)$ to $\mathcal{D}_{t}$. 

\begin{algorithm*}[t]
    \caption{\algoOurs{}: Eliciting User Preferences}
    \begin{algorithmic}[1]
        \State \textbf{Input:} generative model $\pi: \mathcal{X} \to \mathcal{Y}$, feature mapping $\phi: \mathcal{X} \times \mathcal{Y} \to \mathbb{R}^d$, stopping threshold $\epsilon$, confidence level $\delta$, problem dimension $d$, and norm bound $S$
        \State User $u$ provides an input prompt $x \in \mathcal{X}$ to the system.
        \State Generate a diverse set of responses $\mathcal{Y}_\text{cand}$ for $x$ by sampling from $\pi$.
        \State Initialize the preference dataset $\mathcal{D}_0 \gets \bcc{}$.
        \For{$t = 0, 1, 2, \dots$}
        \Statex \ \ \quad \LINECOMMENT{Obtain a representative preference model parameter and confidence set.}
        \State Obtain $(\widehat{\theta}_t, \Theta_t) \gets \textsc{Solve}(\mathcal{D}_{t}, d, S, t, \delta)$
        \Statex \ \ \quad \LINECOMMENT{Obtain responses for comparison.}
        \State Select the first response $y_t^{(1)} \gets \argmax_{y \in \mathcal{Y}_\text{cand}} \langle \widehat{\theta}_t, \phi(x, y) \rangle$.
        \State Select the second response $(y_t^{(2)}, \widetilde{\theta}_t) \gets \argmax_{(y', \theta) \in \mathcal{Y}_\text{cand} \times \Theta_t} \langle \theta, \phi(x, y') - \phi(x, y_t^{(1)}) \rangle$.
        \Statex \ \ \quad \LINECOMMENT{Check the stopping condition.}
        \State Compute stopping criteria $B(t) = \langle \widetilde{\theta}_t, \phi(x, y_t^{(2)}) - \phi(x, y_t^{(1)}) \rangle$.
        \If{$B(t) \leq \epsilon$}
         \State \textbf{Output:} Response $y_t^{(1)}$ to the user $u$.
        \EndIf
        \Statex \ \ \quad \LINECOMMENT{Obtain user feedback and update the preference dataset.}
        \State Ask the user $u$ to provide preference over responses $y_t^{(1)}$ and $y_t^{(2)}$.
        \State Observe $r_t \sim p_u [y_t^{(1)} \succ y_t^{(2)} \mid x]$ and update $\mathcal{D}_{t+1} \gets \mathcal{D}_{t} \cup \{(x, y_t^{(1)}, y_t^{(2)}, r_t)\}$.
        \EndFor{}
    \end{algorithmic}
    \label{alg:general-pad}
\end{algorithm*} 

\begin{algorithm*}[t]
    \caption{\algoOurs{}: \textsc{Solve} Subroutine}
    \begin{algorithmic}[1]
        \State \textbf{Input:} preference dataset $\mathcal{D}_t$,  dimension $d$, norm bound $S$, step $t$, and confidence level $\delta$
        %
        \State Compute the MLE $\widehat{\theta}_t$ by solving the optimization problem in Eq.~\eqref{eq:mle-estimate}.
        \State Construct the loss-based confidence set $\Theta_t$ as in Eq.~\eqref{eq:confidence-set}.
        \Statex \LINECOMMENT{Obtain the practical confidence set (see Section~\ref{subsec:practical-conf-set}).}
        \State {\color{teal} Update $\Theta_t \gets \Theta_t \cap \mathcal{H}_t$ w.r.t. the consistency set in Eq.~\eqref{eq:half-set};}
        \Statex {\color{teal} when updated $\Theta_t$ is empty, set $\Theta_t = \{ \widehat{\theta}_t \}$}.
        \State \textbf{Output:} preference parameter $\widehat{\theta}_t$ and confidence set $\Theta_t$
    \end{algorithmic}
    \label{alg:solve-pad}
\end{algorithm*} 

%
\textbf{\textsc{Solve}.} For the preference dataset $\mathcal{D}_{t}$, we define the negative log-likelihood function as follows: $\mathcal{L}_t(\theta) := \sum_{\tau=0}^{t-1} \ell(\theta; (x_\tau, y_\tau^{(1)}, y_\tau^{(2)}, r_\tau))$, where $\ell(\theta; (x, y^{(1)}, y^{(2)}, r)) := - r \cdot \log \mu(\langle {\theta}, z \rangle) - (1- r) \cdot \log (1 - \mu(\langle {\theta}, z \rangle))$ is the logistic loss function, with $z = {\phi(x, y^{(1)}) - \phi(x, y^{(2)})}$.
Finally, we obtain the norm-constrained, unregularized maximum likelihood estimator (MLE) of the unknown parameter $\theta^\star$ via solving the following optimization problem:
\begin{equation}
\label{eq:mle-estimate}
\widehat{\theta}_t ~:=~ \argmin_{\norm{\theta}_2 \leq S} \mathcal{L}_t(\theta) .
\end{equation}
We construct the loss-based confidence set as follows:
\begin{equation}
\label{eq:confidence-set}
\Theta_t ~:=~ \bcc{\theta \in \mathbb{R}^d: \norm{\theta}_2 \leq S \text{ and } \mathcal{L}_t(\theta) \leq \mathcal{L}_t(\widehat{\theta}_t) + \beta_t} ,
\end{equation}
where $\beta_t = 10 d \log \brr{\frac{S t}{4 d} + e} + 2 ((e - 2) + S) \log \frac{1}{\delta}$~\cite{lee2024improved}. Since $\mathcal{L}_t$ is convex, the confidence set $\Theta_t$ is also convex. The \textsc{Solve} procedure is presented in Algorithm~\ref{alg:solve-pad}.
\looseness-1After obtaining $\widehat{\theta}_t$ and $\Theta_t$, we select the first response $y_t^{(1)}$ as the one that maximizes the win-rate objective in Eq.~\eqref{eq:optimal-aligned-response} under $\widehat{\theta}_t$. Next, we choose the second response $y_t^{(2)}$ to maximize the regret of $y_t^{(1)}$ within $\Theta_t$. If this maximum regret (i.e., $B(t)$ in line~9 of Algorithm~\ref{alg:general-pad}) is below the given $\epsilon$ threshold, we output $y_t^{(1)}$.
%


\begin{figure*}[t]
  \centering

    \begin{minipage}[t]{0.99\textwidth}
      \centering
      \begin{subfigure}[t]{0.92\textwidth}
        \centering
        \includegraphics[trim=0.5mm 0.5mm 0.5mm 0.5mm, clip, width=0.95\textwidth]{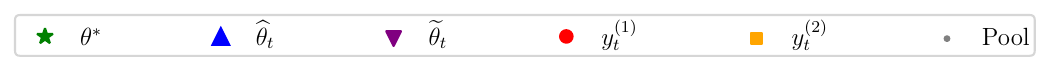}
      \end{subfigure}

    \centering
      \begin{subfigure}[t]{0.195\textwidth}
        \includegraphics[trim=3.5mm 3.5mm 3.5mm 3.5mm, clip, width=0.88\linewidth]{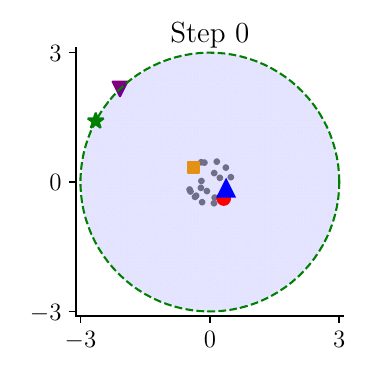}
      \end{subfigure}
      \begin{subfigure}[t]{0.195\textwidth}
        \includegraphics[trim=3.5mm 3.5mm 3.5mm 3.5mm, clip, width=0.88\linewidth]{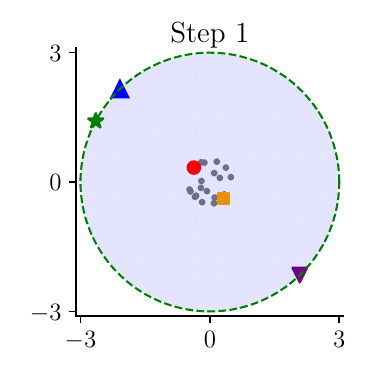}
      \end{subfigure}
      \begin{subfigure}[t]{0.195\textwidth}
        \centering
        \vspace*{-12.5mm}  
        \textbf{\textellipsis \ \ \textellipsis \ \ \textellipsis}
      \end{subfigure}
      \begin{subfigure}[t]{0.195\textwidth}
        \includegraphics[trim=3.5mm 3.5mm 3.5mm 3.5mm, clip, width=0.88\linewidth]{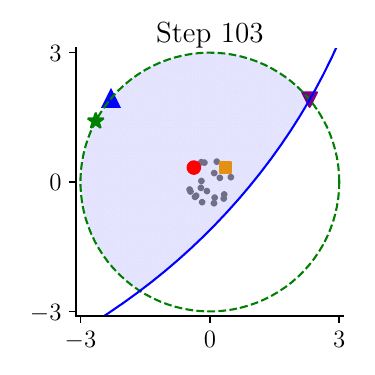}
      \end{subfigure}
      \begin{subfigure}[t]{0.195\textwidth}
        \includegraphics[trim=3.5mm 3.5mm 3.5mm 3.5mm, clip, width=0.88\linewidth]{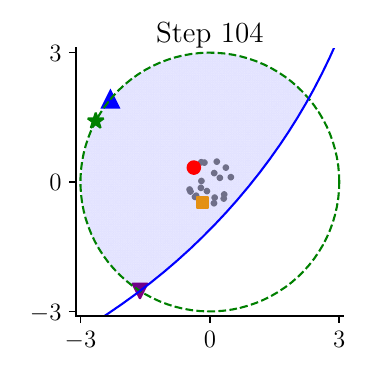}
      \end{subfigure}      
      \vspace{-6mm}
      \caption*{(a) Loss-based confidence set $\Theta_t$ in Eq.~\eqref{eq:confidence-set} shrinks slowly due to incremental log-loss updates.}
    \end{minipage}
    \vspace{3mm}
    
    \begin{minipage}[t]{0.99\textwidth}
    \centering
      \begin{subfigure}[t]{0.195\textwidth}
        \includegraphics[trim=3.5mm 3.5mm 3.5mm 3.5mm, clip, width=0.88\linewidth]{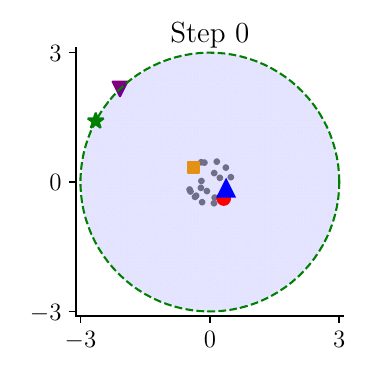}
      \end{subfigure}
      \begin{subfigure}[t]{0.195\textwidth}
        \includegraphics[trim=3.5mm 3.5mm 3.5mm 3.5mm, clip, width=0.88\linewidth]{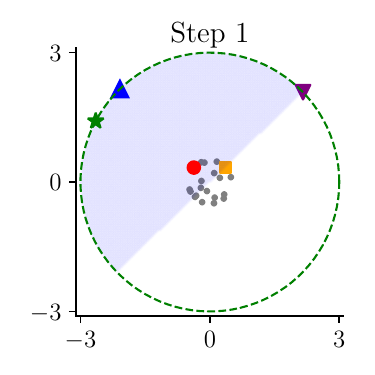}
      \end{subfigure}
      \begin{subfigure}[t]{0.195\textwidth}
        \includegraphics[trim=3.5mm 3.5mm 3.5mm 3.5mm, clip, width=0.88\linewidth]{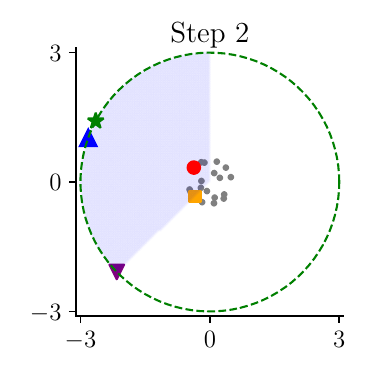}
      \end{subfigure}
      \begin{subfigure}[t]{0.195\textwidth}
        \includegraphics[trim=3.5mm 3.5mm 3.5mm 3.5mm, clip, width=0.88\linewidth]{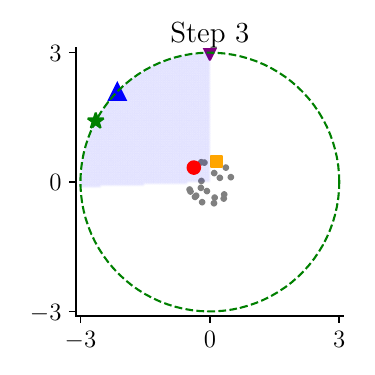}
      \end{subfigure}
      \begin{subfigure}[t]{0.195\textwidth}
        \includegraphics[trim=3.5mm 3.5mm 3.5mm 3.5mm, clip, width=0.88\linewidth]{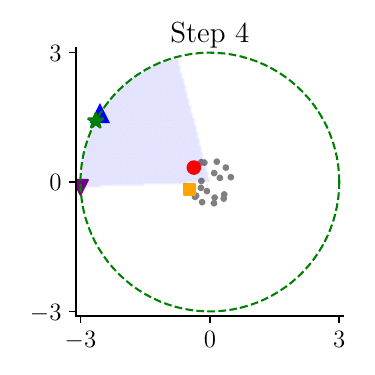}
      \end{subfigure}
      \vspace{-6mm}
      \caption*{\looseness-1(b) Practical confidence set $\Theta_t \cap \mathcal{H}_t$ shrinks aggressively due to version-space elimination.}
    \end{minipage}

    \vspace{3mm}

    \begin{minipage}[t]{0.99\textwidth}
      \begin{subfigure}[t]{\textwidth}
        \centering
        \includegraphics[trim=0.5mm 0.5mm 0.5mm 0.5mm, clip, width=0.95\textwidth]{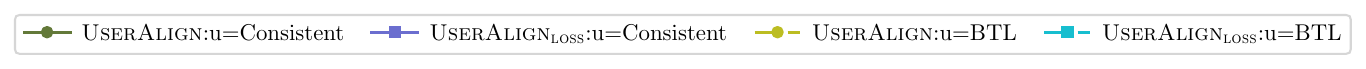}
       \vspace{-5mm}
      \end{subfigure}

    \centering
      \begin{subfigure}[t]{0.31\textwidth}
        \includegraphics[height=2.2cm]{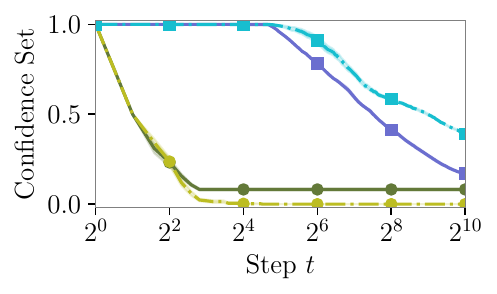}
      \end{subfigure}
      \ \
      \begin{subfigure}[t]{0.31\textwidth}
        \includegraphics[height=2.2cm]{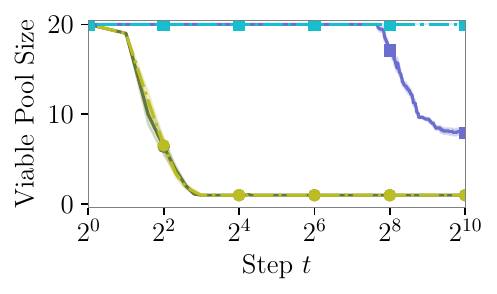}
      \end{subfigure}
      \ \
      \begin{subfigure}[t]{0.31\textwidth}
        \includegraphics[height=2.2cm]{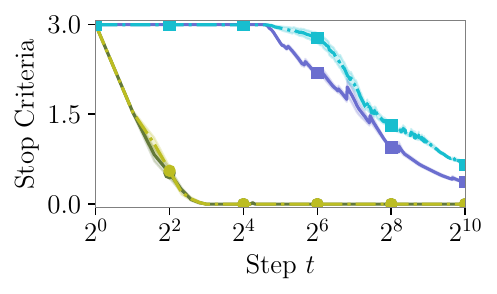}
      \end{subfigure}
      \vspace{-2.5mm}
      \caption*{(c) Normalized area of confidence set, viable response pool size, and stopping criteria over interaction steps. The practical confidence set shrinks significantly faster.}
    \end{minipage}
  \caption{Geometric convergence in a two-dimensional synthetic domain. We consider a 2D preference space where each candidate response is represented by a point randomly sampled from the ball of radius $0.5$. For each run, the ground-truth user preference $\theta^\star$ is sampled uniformly from the circle with radius $3$. The plots compare the rate at which the loss-based and practical confidence sets shrink over successive pairwise comparisons.}
  \label{fig.synth_2d}
  \vspace{-1mm}
\end{figure*}

\textbf{Theoretical Analysis.} The following proposition shows that the loss-based confidence set $\Theta_t$ contains the true parameter $\theta^\star$ with high probability~\cite{lee2024improved}. Proofs are provided in the supplementary material.
\begin{proposition}
\label{prop:conf}
For the confidence set $\Theta_t$ defined in Eq.~\eqref{eq:confidence-set}, we have: $\Prob{\forall t \geq 0, \theta^\star \in \Theta_t} ~\geq~ 1 -\delta$.
\end{proposition}
Note that for any $\theta'$, $\mathcal{L}_t(\theta) - \mathcal{L}_t(\theta') \leq \mathcal{L}_t(\theta) - \mathcal{L}_t(\widehat{\theta}_t) \leq \beta_t$. Therefore, even if only an approximate estimate of $\widehat{\theta}_t$ is obtained, the high-probability guarantee that $\theta^\star \in \Theta_t$ still holds. The following theorem shows that \algoOursLoss{} identifies an $\epsilon$-near-optimal response with high probability.
\begin{theorem}
\label{thm:pac}
Let $\tau$ be the stopping round of \algoOursLoss{}, and $y^\star$ is defined in Eq.~\eqref{eq:optimal-aligned-response}. Then, the response $y_\tau^{(1)}$ returned by \algoOursLoss{} satisfies $\mathbb{P}[\langle \theta^\star, \phi(x, y^\star) - \phi(x, y_\tau^{(1)}) \rangle \leq \epsilon] ~\geq~ 1 - \delta$.
\end{theorem}
Below, we show that the stopping time of \algoOursLoss{} is bounded with high probability.
\begin{theorem}
\label{thm:stop-bound-2}
Let $\tau$ be the stopping round of \algoOursLoss{}. Define $K = \abs{\mathcal{Y}_\text{cand}}$, and $\Omega : = \frac{S^2 \kappa^\star_{\mathcal{X}, \mathcal{Y}}}{\max\{\epsilon, \Delta_{\mathcal{Y}_\text{cand}}\}^2} \brr{d + \log \frac{1}{\delta}}$. Then, with probability at least $1-\delta$, we have: $\tau \leq \mathcal{O}(\Omega K^2 \cdot \log (\Omega K^2))$.
\end{theorem}

\subsection{Practical Confidence Set in \algoOurs{}}
\label{subsec:practical-conf-set}

Despite strong convergence guarantees, \algoOursLoss{} requires a larger number of user preference queries to identify the best (win-rate maximizing) response in practice. This is because the loss-based confidence set defined in Eq.~\eqref{eq:confidence-set} shrinks slowly due to incremental log-loss updates (see Figure~\ref{fig.synth_2d}). This behavior is expected under the stochastic BTL model of user preferences. To overcome this challenge, the key idea is to treat the user's preference feedback as consistent and noise-free\footnote{Our work focuses on short-term, task-specific interaction sessions, where 
users typically have clear and stable preferences. In such cases, assuming 
consistent and noise-free feedback is both practical and realistic.}, and incorporate it into the theoretical framework to identify the best response more efficiently. Specifically, by leveraging version-space elimination via intersecting halfspaces consistent with observed preference tuples~\cite{DBLP:journals/ml/Littlestone87,DBLP:books/daglib/0033642}, we can aggressively shrink the confidence set. Given the preference dataset $\mathcal{D}_{t} = \{(x_\tau, y_\tau^{(1)}, y_\tau^{(2)}, r_\tau)\}_{\tau = 0}^{t-1}$, we define the consistent half-spaces set as:
\begin{equation}
\label{eq:half-set}
\mathcal{H}_t ~:=~ \bcc{\theta \in \mathbb{R}^d: r_\tau \cdot \langle \theta, z_\tau \rangle - (1 - r_\tau) \cdot \langle \theta, z_\tau \rangle \geq 0 \text{ for all } \tau \in [t-1]} ,
\end{equation}
where $z_\tau = {\phi(x_\tau, y_\tau^{(1)}) - \phi(x_\tau, y_\tau^{(2)})}$. Using this consistent set of halfspaces, we refine the loss-based confidence set as $\Theta_t \gets \Theta_t \cap \mathcal{H}_t$, enabling rapid identification of a near-optimal response with fewer queries; in case $\Theta_t$ becomes empty, we set it as $\Theta_t = \{ \widehat{\theta}_t \}$. We refer to the resulting algorithm with the updated confidence set as our main method, \algoOurs{}. The computational efficiency of \algoOurs{} is discussed in Appendix~\ref{sec-app.comp-eff}. In the following section, we empirically demonstrate the effectiveness of \algoOurs{} in quickly selecting personalized responses.

%

\section{Experimental Evaluation}
\label{sec.experiments}
\vspace{-1mm}

To thoroughly evaluate our method, we consider a diverse set of domains (see Table~\ref{fig.domains}). We begin with \texttt{food2d}, enabling controlled experiments with users modeled through BTL in an interpretable 2D space. We then assess real-world domains (\texttt{food64d}, \texttt{travel64d}, \texttt{visual512d}) to test performance with complex GPT-simulated personas. Finally, we include \texttt{dsp64d} based on an existing large-scale benchmark to rigorously examine scalability and robustness.

\looseness-1\textbf{Domain \texttt{food2d}}. This domain defines $\Theta$ as a 2D space with dimensions `spiciness' and `veginess', each feature ranging from $-1$ (no spice/animal protein) to $1$ (high spice/plant-based protein). Building on this space, we define the set of questions, construct user models, generate response pools, and obtain feature representations. First, we construct $\mathcal{X}$ as $10$ food recommendation questions. Second, for each question, we sample $3$ user preferences $\theta^\star$ uniformly on the circle of radius $S = 3$ in $\mathbb{R}^2$ (i.e., $\norm{\theta^\star}_2 = 3$). We conduct experiments with two types of user behavior: $p_u$ modeled through BTL, and $p_u$ considered consistent and noise-free. Third, we generate a candidate pool $\mathcal{Y}_\text{cand}$ of $20$ responses per question  $x \in \mathcal{X}$ with GPT-4o~\cite{GPT4o}. Finally, we obtain $\phi(x, y)$ by asking GPT-4o to map a question-response pair to the 2D `spiciness'/`veginess' space, given the domain description as context.

\begin{table*}
    \caption{Domain specifications for our experiments involving personalized text and image generation.}
    \centering
    \scalebox{0.78}{
    \setlength\tabcolsep{3pt}
    \renewcommand{\arraystretch}{1.1}
        \begin{tabular}{lm{3.63cm}cccm{5.8cm}}
            \toprule

            \multicolumn{1}{l}{\textbf{Domain}} &
            \multicolumn{1}{l}{\textbf{Preference Space}} &
            \multicolumn{1}{c}{\textbf{\# Questions}} &
            \multicolumn{1}{c}{\textbf{\# Users}} &
            \multicolumn{1}{c}{\textbf{Response Pool}} &
            \multicolumn{1}{c}{\textbf{Example Question}} \\

            \midrule
            
            \texttt{food2d} & %
            2d (domain-specific)
            & \phantom{0}10 & 3 & 20 & What should I cook for dinner tonight? \\
            
            \midrule
            \texttt{food64d} & %
            64d (Potion \cite{minishlab2024model2vec})
            & \phantom{0}10 & 3 & 20 & What should I cook for dinner tonight? \\

            \texttt{travel64d} & %
            64d (Potion \cite{minishlab2024model2vec})
            & \phantom{0}10 & 3 & 20 & Which travelling destination would surprise me in the best possible way? \\
            
            \texttt{visual512d} & %
            512d (OpenCLIP \cite{ilharco_gabriel_2021_5143773, datacomp})
            & \phantom{0}10 & 3 & 40 & Generate concept art for a videogame hero that players would instantly connect with. \\
            
            \midrule
            \texttt{dsp64d} & %
            64d (Potion \cite{minishlab2024model2vec})
            & 100 & 3 & 20 & Describe the main character of Shakespeare's play Hamlet. \\
            \bottomrule
        \end{tabular}
    }

    \label{fig.domains}
    \vspace{-2mm}
\end{table*}

\looseness-1\textbf{Domains \texttt{food64d}, \texttt{travel64d}, and \texttt{visual512d}}. We define $\Theta$ as $\{\theta \in \mathbb{R}^d : \norm{\theta}_2 \leq S\}$, with dimensionality $d$ based on pretrained embeddings and  $S=3$. First, we define $\mathcal{X}$ as $10$ questions per domain (see Table~\ref{fig.domains}). Second, for each question, we construct $3$ brief persona descriptions representing user preferences, and simulate user behavior $p_u$ by prompting GPT-4o-mini~\cite{GPT4omini} ($\text{temperature}=0$) to select preferred responses conditioned on these descriptions (see Figure~\ref{fig.introduction}). Third, we construct $\mathcal{Y}_\text{cand}$ by generating $20$ responses per question with GPT-4o for text domains (\texttt{food64d}, \texttt{travel64d}) and $40$ images per question with GPT-Image-1~\cite{GPTImage} for \texttt{visual512d}. Finally, we obtain $\phi(x, y)$ using pretrained models, specifically a sentence transformer for the 64-dimensional text domains~\cite{minishlab2024model2vec} and an OpenCLIP variant for the 512-dimensional vision domain~\cite{ilharco_gabriel_2021_5143773, Radford2021LearningTV, datacomp}. The framework is modular, so specialized (handcrafted or learned) embeddings can replace the pretrained ones to capture finer nuances when needed.

\textbf{Domain \texttt{dsp64d}}. For this domain, we similarly define $\Theta$ with dimensionality based on pretrained embeddings and $S=3$. First, we define $\mathcal{X}$ as $100$ questions sampled from the Domain Specific Preference (DSP) dataset~\cite{DBLP:journals/corr/abs-2309-03126, zhu2025onthefly}. Second, for each question, we sample $3$ user descriptions from the dataset to simulate user behavior $p_u$, prompting GPT-4o-mini to select preferred responses given each description. Third, we construct the candidate pool $\mathcal{Y}_\text{cand}$ for each question using GPT-4o, following the same approach as above. Finally, we obtain $\phi(x, y)$ using the same sentence transformer encoder as for the other text domains, yielding 64-dimensional representations.

%


\begin{figure*}[t]
  \centering

    \begin{subfigure}[t]{\textwidth}
      \centering
      \includegraphics[trim=0.5mm 2mm 0.5mm 2mm, clip, width=0.82\textwidth]{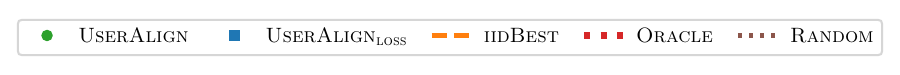}
    \end{subfigure}
    \vspace{-5mm}
    
    \begin{subfigure}[t]{0.48\textwidth}
      \includegraphics[height=3cm]{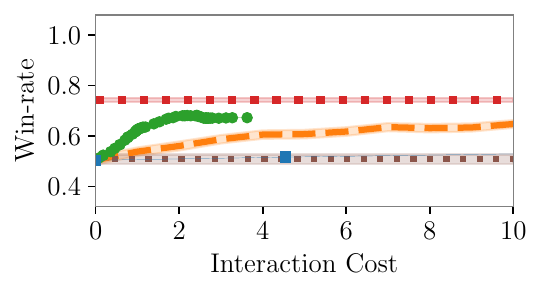}
      \vspace{-3mm}
      \caption{$u=\text{BTL}$: Win-rate vs. interaction cost}      
      \label{fig.results_foods2d.cost.btl}
    \end{subfigure}
    \ \ \ \ 
    \begin{subfigure}[t]{0.48\textwidth}
      \includegraphics[height=3cm]{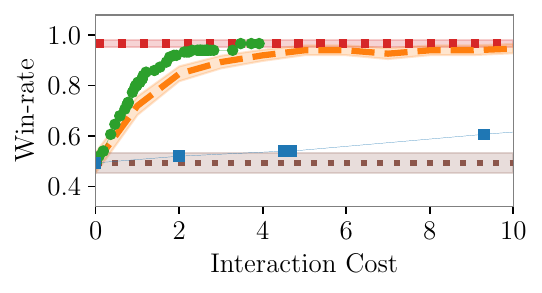}
      \vspace{-3mm}      
      \caption{$u=\text{Consistent}$: Win-rate vs. interaction cost}     
      \label{fig.results_foods2d.cost.det}
    \end{subfigure}

    \vspace{1.5mm}
    \begin{subfigure}[t]{\textwidth}
      \centering
      \includegraphics[trim=0.5mm 2mm 0.5mm 2mm, clip, width=0.82\textwidth]{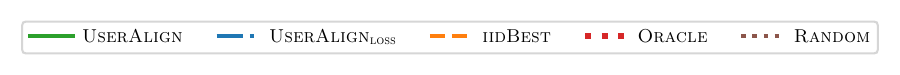}
    \end{subfigure}
    \vspace{-5mm}
    
    \begin{subfigure}[t]{0.48\textwidth}
      \includegraphics[height=3cm]{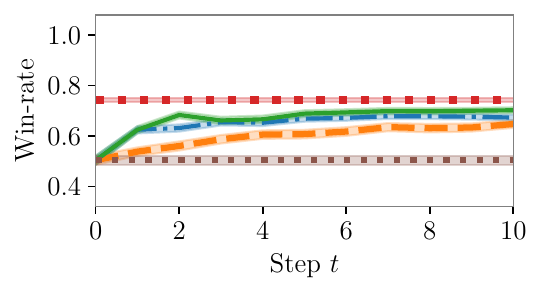}
      \vspace{-3mm}
      \caption{$u=\text{BTL}$: Win-rate w.r.t. increasing steps}
      \label{fig.results_foods2d.performance.btl}
    \end{subfigure}
    \ \ \ \ 
    \begin{subfigure}[t]{0.48\textwidth}
      \includegraphics[height=3cm]{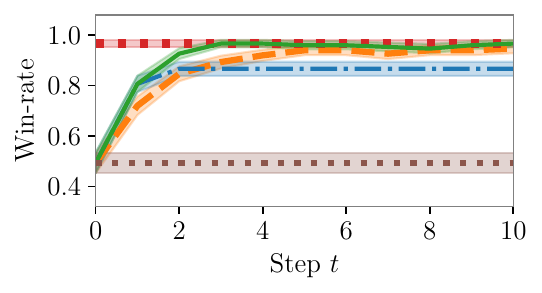}
      \vspace{-3mm}
      \caption{$u=\text{Consistent}$: Win-rate w.r.t. increasing steps} 
      \label{fig.results_foods2d.performance.det}
    \end{subfigure}
    
  \vspace{-1mm}
  \caption{\looseness-1Results on \texttt{food2d} across two types of user behavior. Top row shows win-rate vs. interaction cost trade-off; here \algoOurs{} and \algoOursLoss{} results show scatter plot corresponding to varying values of $\epsilon$. Bottom row shows win-rate per increasing interaction steps; here \algoOurs{} and \algoOursLoss{} are run for a given number of steps as mentioned in Footnote~\ref{footnote.algo_no_stopping}. In these plots, \textsc{Oracle} and \textsc{Random} are flat lines, and \textsc{iidBest} results are reported at different number of given steps.}
  \vspace{-3mm}
  \label{fig.results_foods2d}
\end{figure*}

\vspace{-1mm}
\subsection{Evaluation Setup}
\label{sec.experiments.setup}
\vspace{-1mm}
\textbf{Candidate pool generation.} For each question, we generate a diverse pool of candidate responses by following two sampling approaches. We obtain half of the candidate responses by sampling from GPT-4o or GPT-Image-1 at high temperature, conditioned only on the question, resulting in unbiased responses. We obtain the remaining half by prompting the generative model to first reason about possible diverse interests relevant to the question, then generate one response per interest, thereby obtaining greater diversity in the pool. We provide full details about pool generation in the supplementary material, and also report additional experiments using different candidate pools.

\looseness-1\textbf{Metrics.} We consider two metrics: (i) \textit{interaction cost} as the number of interaction steps taken by a given method before outputting a final response; (ii) \textit{win-rate} $p_u[y \succ \widetilde{y} \mid x]$ of the output response $y$ against baseline $\widetilde{y}$, where $\widetilde{y}$ is generated by the original model (GPT-4o or GPT-Image-1) with zero sampling temperature for the same question $x$ (see Section~\ref{sec.problemsetup}). The win-rate is computed using different types of simulated user preference models $p_u$, with $u=\text{BTL}$, $u=\text{Consistent}$, or $u=\text{GPT}$.

\vspace{-1mm}
\subsection{Evaluated Methods}
\vspace{-1mm}
\looseness-1\textbf{\textsc{Oracle} and \textsc{Random}.} These two baseline methods provide upper/lower performance bounds. \textsc{Oracle} selects the response with highest utility using explicit knowledge of the user. \textsc{Random} selects a response uniformly at random from the candidate pool. These methods don't have any interaction cost.

\looseness-1\textbf{\textsc{iidBest}.} This method collects user preferences over a fixed number of steps, each time randomly sampling a pair of responses. At the end, it computes the MLE $\widehat{\theta}$ of the user's preference as in Eq.~\eqref{eq:mle-estimate}, and selects the response maximizing the utility, i.e., $\widehat{y} = \argmax_{y \in \mathcal{Y}_\text{cand}} \langle \widehat{\theta}, \phi(x, y) \rangle$.

\looseness-1\textbf{\algoOurs{} and \algoOursLoss{}.} These two methods select pairs using Algorithm~\ref{alg:general-pad}. Our main method, \algoOurs{}, is based on the practical confidence set introduced in Section~\ref{subsec:practical-conf-set}. \algoOursLoss{} is based on the theoretical confidence set, introduced in Section~\ref{sec:theoretical-analysis}. Both methods are parameterized by $(\epsilon, \delta)$ which determines their stopping condition (see lines~1~and~10 in Algorithm~\ref{alg:general-pad}). We set $\delta=0.05$ for all the experiments. When reporting results, we will vary the value of $\epsilon$ in $[0, S]$ to get a trade-off between interaction cost and win-rate.\footnote{\looseness-1We will additionally look at dynamics when running these algorithms for a fixed number of steps without using $\epsilon$-based stopping; here, the algorithm will resort to i.i.d. sampling of $y^{(2)}$ if no viable response remains.\label{footnote.algo_no_stopping}}

\subsection{Evaluation Results}
\label{sec.experiments.results}
\vspace{-1mm}


\begin{figure*}[t]
  \centering

    \begin{subfigure}[t]{\textwidth}
      \centering
      \includegraphics[trim=0.5mm 2mm 0.5mm 2mm, clip, width=0.82\textwidth]{figs/5_experiments/rq2/legend_reward_vs_cost_plot_foods_2d.pdf}
    \end{subfigure}

    \begin{subfigure}[t]{0.48\textwidth}
      \includegraphics[height=3cm]{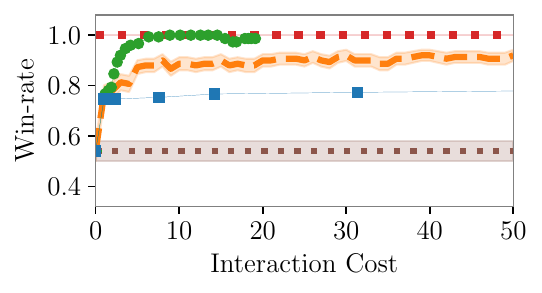}
      \vspace{-2.5mm}
      \caption{Domain \texttt{food64d}}
    \end{subfigure}
    \ \ \ \  
    \begin{subfigure}[t]{0.48\textwidth}
      \includegraphics[height=3cm]{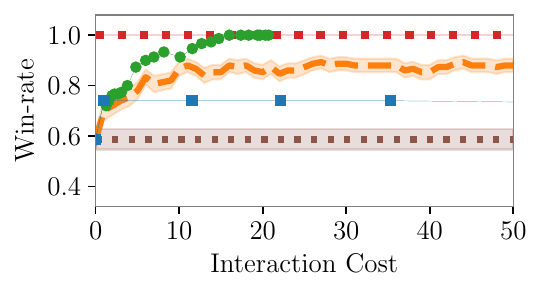}
      \vspace{-2.5mm}
      \subcaption{Domain \texttt{travel64d}}
    \end{subfigure}
    
    \begin{subfigure}[t]{0.48\textwidth}
      \includegraphics[height=3cm]{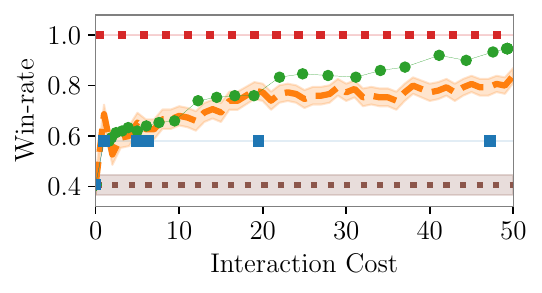}
      \vspace{-2.5mm}
      \subcaption{Domain \texttt{visual512d}}
    \end{subfigure}
    \ \ \ \  
    \begin{subfigure}[t]{0.48\textwidth}
      \includegraphics[height=3cm]{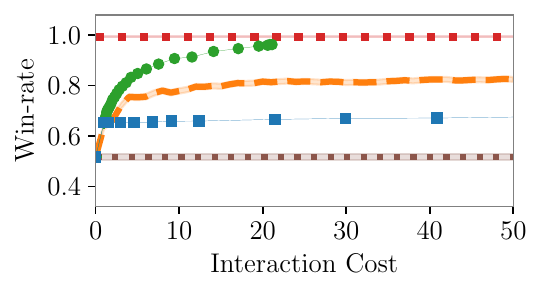}
      \vspace{-2.5mm}
      \caption{Domain \texttt{dsp64d}}
  \end{subfigure}

  \caption{\looseness-1Win-rate vs. interaction cost trade-off on domains with arbitrary preference spaces. Here \algoOurs{} and \algoOursLoss{} results show scatter plot corresponding to varying values of $\epsilon$. \textsc{Oracle} and \textsc{Random} are flat lines, and \textsc{iidBest} results are reported at different number of given steps.
  }
  \label{fig.four_domains_half}
  \vspace{-4.5mm}
\end{figure*}

All results are averaged over five random seeds, questions, and user personas (see Table~\ref{fig.domains}).

\textbf{Results w.r.t. different types of users in 2D.} Figure~\ref{fig.results_foods2d} shows the results on \texttt{food2d} across two types of user behaviors ($u=\text{BTL}$ and $u=\text{Consistent}$). \algoOurs{} achieves high win-rate at a low interaction cost, outperforming \algoOursLoss{}, \textsc{iidBest}, and \textsc{Random}. Moreover, Figures~\ref{fig.results_foods2d.cost.btl}~ and~\ref{fig.results_foods2d.cost.det} showcase that \algoOurs{} is effective in automatically deciding how many interaction steps are needed to achieve competitive win-rates based on its stopping criterion.

\textbf{Results in an arbitrary preference space.} Figure~\ref{fig.four_domains_half} shows the results on four domains with arbitrary preference spaces, where the user preferences are given by GPT-based simulated personas.  
\algoOurs{} achieves competitive win-rates at lower interaction costs, outperforming \algoOursLoss{}, \textsc{iidBest}, and \textsc{Random} across all domains. As an illustrative example, Figure~\ref{fig.introduction} highlights that \algoOurs{} picks informative comparison queries, and can identify a high-quality response quickly even in high-dimensional settings. Overall, these findings highlight the robustness and adaptability of our method across diverse domains and representation spaces.
%

\section{Evaluation with Human Users}
\label{sec.userstudy}

Next, to go beyond simulated user preference models considered above ($u=\mathrm{BTL}$, $u=\mathrm{Consistent}$, $u=\mathrm{GPT}$), we evaluate the methods with an additional user behavior type, $u=\mathrm{Human}$, with preferences coming from human users.

\textbf{Evaluation setup.}
We consider two domains, \texttt{food64d} and \texttt{visual512d}, due to the high cost involved in this evaluation.
We compare three methods (\algoOurs{}, \textsc{iidBest}, and \textsc{Random}) using the same candidate pool sizes as in the earlier experiments (\texttt{food64d} has $20$ responses per question and \texttt{visual512d} has $40$ images per question).
To ensure a fair comparison and keep cognitive load manageable, we fix the per-domain interaction budget to $10$ for \texttt{food64d} and $20$ for \texttt{visual512d}, and we report results at interaction steps $t \in \{5, 10\}$ for \texttt{food64d} and at $t \in \{10, 20\}$ for \texttt{visual512d}.
Personalization is varied with two conditions to test robustness to how preferences are specified. In the with-persona condition, participants see a concise persona and are asked to roleplay it (similar to Section~\ref{sec.experiments}). In the without-persona condition, no persona is shown and participants choose according to their own preferences, reflecting greater individual variability and a more challenging setting. Method identities are blinded to the users throughout.

\textbf{Web application and participation sessions.}
We provide an overview of the user study setup below, and full details are available in Appendix~\ref{sec-app:userstudy}. We developed a web application to expose the methods through an interactive interface. We recruited a total of $960$ participants on Amazon Mechanical Turk, split uniformly across the domains, methods, and personalization conditions. Before a participation session began, each participant was randomly assigned a domain (\texttt{food64d} or \texttt{visual512d}), a question from that domain's question set, an interaction method (\algoOurs{} or \textsc{iidBest}), and a personalization condition (with-persona or without-persona). Each session has three stages. Stage 1 presents the question and the assigned personalization instruction. Stage 2 consists of the fixed number of pairwise comparisons based on the interaction budget mentioned above. In Stage 3, participants compare three final candidates against the zero-temperature baseline for their assigned question: the method's selection at the later interaction step, the selection at the earlier interaction step, and a candidate chosen by the \textsc{Random} method. On average, a session lasted about $7.5$ minutes and participants received a compensation of $1.80~\text{USD}$ for each session.

\textbf{Results.}
In each domain, for each of \algoOurs{} and \textsc{iidBest}, we collected $n=120$ evaluation sessions per setting; for \textsc{Random}, which does not depend on interaction steps, we recorded $2n=240$ samples.  Table~\ref{tab:userstudy} reports win-rates across both domains and personalization conditions. Under matched interaction budgets, \algoOurs{} consistently outperforms \textsc{iidBest}, with larger gains at the later interaction step in each domain and higher absolute performance in the with-persona condition. These results highlight that the improvements observed for the simulated user behavior in Section~\ref{sec.experiments} also carry over when user preferences are provided by humans.

\begin{table*}[t]
    \caption{\looseness-1Results of the \texttt{food64d} and \texttt{visual512d} evaluation with human users under fixed interaction budgets. The table reports win-rate versus the zero-temperature baseline at two different interaction steps for two personalization conditions.}

    \label{tab:userstudy}
    \centering
    \scalebox{0.84}{
        \setlength\tabcolsep{5pt}
        \renewcommand{\arraystretch}{1.12}
        \begin{tabular}{lcccccccc}
            \toprule
            & \multicolumn{4}{c}{\texttt{food64d}} & \multicolumn{4}{c}{\texttt{visual512d}} \\
            \cmidrule(lr){2-5}\cmidrule(lr){6-9}
            \textbf{Method} & \multicolumn{2}{c}{\textbf{With-persona}} & \multicolumn{2}{c}{\textbf{Without-persona}} & \multicolumn{2}{c}{\textbf{With-persona}} & \multicolumn{2}{c}{\textbf{Without-persona}} \\
            \cmidrule(lr){2-3}\cmidrule(lr){4-5}\cmidrule(lr){6-7}\cmidrule(lr){8-9}
            & $t=5$ & $t=10$ & $t=5$ & $t=10$ & $t=10$ & $t=20$ & $t=10$ & $t=20$ \\
            \midrule
            \textsc{Random}   & $44.2$ ($3.2$) & $44.2$ ($3.2$) & $43.8$ ($3.2$) & $43.8$ ($3.2$) & $46.9$ ($3.2$) & $46.9$ ($3.2$) & $42.3$ ($3.2$) & $42.3$ ($3.2$) \\
            \textsc{IIDBest}  & $71.7$ ($4.1$) & $78.3$ ($3.8$) & $73.3$ ($4.0$) & $76.7$ ($3.9$) & $75.6$ ($3.9$) & $81.5$ ($3.6$) & $78.3$ ($3.8$) & $79.2$ ($3.7$) \\
            \algoOurs{}       & $82.5$ ($3.5$) & $89.2$ ($2.8$) & $79.2$ ($3.7$) & $85.8$ ($3.2$) & $81.7$ ($3.5$) & $90.0$ ($2.7$) & $79.8$ ($3.7$) & $82.4$ ($3.5$) \\
            \bottomrule
        \end{tabular}
    }
\end{table*}

\vspace{-1.5mm}
\section{Concluding Discussions}\label{sec.conclusion}
\vspace{-1mm}
We introduced \textsc{UserAlign}, a novel inference-time method for efficiently aligning generative model responses to user preferences via sequential pairwise comparisons. Through theoretical analysis and empirical evaluation across text and image generation domains, we demonstrated the efficacy of our method, which provides substantial speed-ups and cost reductions. %

\looseness-1Next, we discuss a few limitations of our work and outline a future plan to address them.  First, we considered pairwise comparisons as the feedback modality; it would be useful to explore richer forms of user feedback such as rankings or free-form inputs. Second, our method relies on a pre-generated response pool; it would be interesting to study how we can guide the generations adaptively during the preference elicitation process. Finally, our method treats a user's question independently; it would be useful to leverage historical user preferences data, e.g., from previous questions, to reduce interaction cost. %
As broader implications of our work, we note that adopting personalized alignment methods in AI systems raises ethical considerations, including transparency in how preferences are collected/used and guarding against potential bias or overfitting to noisy feedback. Future deployment of such methods in real-world would require proactive safeguards to ensure that personalization enhances user experience without compromising ethical aspects.

\begin{ack}
\looseness-1Nachiket Kotalwar did this work during an internship at the Max Planck Institute for Software Systems (MPI-SWS), Germany. Funded/Co-funded by the European Union (ERC, TOPS, 101039090). Views and opinions expressed are however those of the author(s) only and do not necessarily reflect those of the European Union or the European Research Council. Neither the European Union nor the granting authority can be held responsible for them.
\end{ack}

    \bibliographystyle{unsrt}
    \bibliography{main}
    \clearpage       
    \clearpage    
\appendix
{
    \allowdisplaybreaks
\section{Table of Contents}
\label{sec-app:toc}

In this section, we briefly describe the content provided in the paper's appendices.

\begin{itemize}
    \item Section~\ref{sec-app.food2d} provides further results for \texttt{food2d}, showcasing a worked example of personalized alignment and evaluations with GPT-based simulated users.
    \item Section~\ref{sec-app.tabular} presents win-rate vs. cost trade-off results in tabular form, then offers a brief head-to-head comparison of the methods.
    \item Section~\ref{sec-app.experiments} gives further details about the procedures for generating candidate response pools and reports experiments with unbiased response pools.
    \item \looseness-1Section~\ref{sec-app:userstudy} provides additional details about the web-application used for human evaluation.
    \item Section~\ref{sec-app.costs}  provides a breakdown of API usage, computational efficiency, and costs.
    \item Section~\ref{sec-app:proofs}  contains complete proofs for the theoretical results in the main paper.
\end{itemize}
    \clearpage
\section{Additional Results for \texttt{food2d} domain}
\label{sec-app.food2d}

\setlength{\parindent}{0pt}

\tcbset{
    outputbox/.style={
        colback=gray!2,
        colframe=gray,
        boxrule=0.5pt,
        arc=1mm,
        before skip=2pt,
        after skip=2pt,
        fontupper=\footnotesize,
        left=0mm, right=0mm, top=0mm, bottom=0mm
    },
    question/.style={
        colback=gray!2,
        colframe=gray,
        boxrule=0.5pt,
        arc=1mm,
        top=0.0mm,
        bottom=0.0mm,
        left=0mm,
        right=0mm,
        before skip=2pt,
        after skip=2pt,
        fontupper=\normalsize
    },
    generations/.style={
        colback=gray!2,
        colframe=gray,
        boxrule=0.5pt,
        arc=1mm,
        top=0.0mm,
        bottom=0.0mm,
        left=0mm,
        right=0mm,
        before skip=2pt,
        after skip=2pt,
        fontupper=\scriptsize
    },
    choicebox/.style={
        colback=gray!2,
        colframe=gray,
        boxrule=0.5pt,
        arc=1mm,
        top=0mm,
        bottom=0mm,
        left=0.7mm,
        right=0.7mm,
        before skip=2pt, after skip=2pt,
        fontupper=\normalsize
    },
    imagecell/.style={
        colback=white,
        colframe=white,
        boxrule=0pt,
        sharp corners=south,
        left=-3mm, right=0mm, top=-1.5mm, bottom=-5mm,
    },
    pairtext/.style={
        colback=white,
        colframe=gray!60,
        boxrule=0.4pt,
        arc=1mm,
        fontupper=\scriptsize,
        left=0.5mm, right=0.5mm,
        top=0.5mm, bottom=0.5mm
    },
    chosentext/.style={
        colback=green!10!white,
        colframe=green!60!black,
        boxrule=1.3pt,
        arc=1mm,
        fontupper=\scriptsize,
        left=0.5mm, right=0.5mm,
        top=0.5mm, bottom=0.5mm
    }
}

\begin{figure*}[htbp!]
    \begin{tcolorbox}[question]
    \circled{1} Input: What should I cook for dinner tonight?
    \end{tcolorbox}
    
    \begin{tcolorbox}[generations]
    \circled{2} {\normalsize System generates a diverse pool of responses:}
    \begin{enumerate}[leftmargin=4mm, labelsep=0.3em, itemsep=0ex]
        \item How about trying a creamy garlic chicken pasta with a side of roasted vegetables?
        \item How about trying a garlic butter shrimp pasta with a fresh green salad on the side?
        \item How about trying a creamy garlic Tuscan chicken with spinach and sun-dried tomatoes, served over pasta or rice?
        \item How about trying a delicious and easy-to-make chicken stir-fry with vegetables and rice?
        \item How about trying a delicious chicken stir-fry with vegetables and rice for dinner tonight?
        \item How about trying a creamy garlic chicken with sautéed spinach and mashed potatoes for a comforting and delicious meal?
        \item How about a comforting chicken stir-fry with vegetables and rice for a quick and delicious dinner?
        \item How about trying a creamy garlic parmesan chicken with roasted vegetables on the side?
        \item How about trying a creamy garlic parmesan chicken with roasted vegetables and a side of mashed potatoes?
        \item How about trying a creamy garlic parmesan chicken with roasted vegetables on the side?
        \item How about a spicy chickpea and vegetable curry with a medley of bell peppers, zucchini, and spinach, served over basmati rice?
        \item How about a comforting vegetable risotto with fresh herbs like basil and parsley, where the creamy arborio rice perfectly ...
        \item How about a spicy roasted vegetable stir-fry with a hint of chili, served over quinoa for a nutritious and flavorful plant-based meal?
        \item How about a comforting vegetable risotto with creamy Arborio rice, sautéed mushrooms, and tender asparagus, finished with a sprinkle ...
        \item How about a hearty roast chicken with garlic mashed potatoes and a side of steamed broccoli for a satisfying, protein-rich meal?
        \item How about a juicy ribeye steak with a side of garlic butter mushrooms for a rich, savory meal?
        \item How about sizzling chicken fajitas with a touch of jalapeño for that perfect hint of heat, served alongside a vibrant mix of ...
        \item How about a spicy chicken stir-fry with bell peppers and a hint of garlic, served over jasmine rice for a satisfying yet mildly fiery meal?
        \item Ignite your taste buds with a fiery Thai green curry loaded with fresh vegetables and a generous splash of chili heat.
        \item How about a comforting vegetable stir-fry with broccoli, bell peppers, and tofu, served over fluffy jasmine rice, seasoned with ...
    \end{enumerate}
    \end{tcolorbox}

    \begin{tcolorbox}[choicebox]
    \circled{3} {\algoOurs{} eliciting user preferences via pairwise comparisons:}
    \vspace{-2mm}

    \begin{center}
        \includegraphics[width=0.9\textwidth]{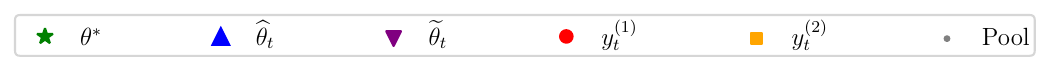}
    \end{center}
    
    \vspace{-3mm}
    
    \noindent
    \begin{minipage}{0.49\textwidth}
        \begin{tcolorbox}[imagecell]
            \centering
            \includegraphics[trim={0 0 0 3mm}, clip,width=0.6\textwidth]{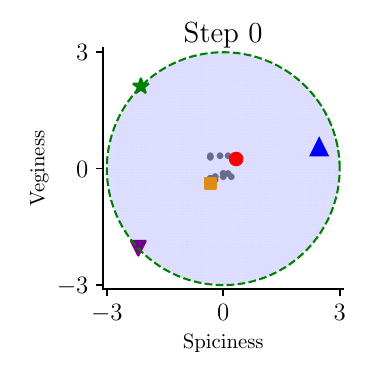}
        \end{tcolorbox}
        \begin{tcolorbox}[chosentext, top=1.7mm, bottom=1.7mm]
            Ignite your taste buds with a fiery Thai green curry loaded with fresh vegetables and a generous splash of chili heat.
        \end{tcolorbox}
        \vspace{-3mm}
        \begin{tcolorbox}[pairtext]
            How about a juicy ribeye steak with a side of garlic butter mushrooms for a rich, savory meal?
        \end{tcolorbox}
    \end{minipage}
    \hfill
    \begin{minipage}{0.49\textwidth}
        \begin{tcolorbox}[imagecell]
            \centering
                \includegraphics[trim={0 0 0 3mm}, clip,width=0.6\textwidth]{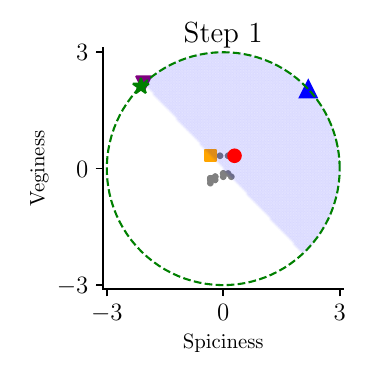}
        \end{tcolorbox}
        \begin{tcolorbox}[pairtext]
            How about a spicy chickpea and vegetable curry with a medley of bell peppers, zucchini, and spinach, served over basmati rice?
        \end{tcolorbox}
        \vspace{-3mm}
        \begin{tcolorbox}[chosentext]
            How about a comforting vegetable risotto with fresh herbs like basil and parsley, where the creamy arborio rice perfectly complements the flavors of sautéed mushrooms, peas, and asparagus?
        \end{tcolorbox}
    \end{minipage}
    
    \end{tcolorbox}

    \begin{tcolorbox}[outputbox]
        \circled{4} Output \\
        \noindent
        \begin{minipage}{0.27\textwidth}
            \begin{tcolorbox}[imagecell, boxrule=0pt, left=-2mm, right=-2mm, top=-1mm, bottom=-1mm]
                \centering
                \includegraphics[trim={0 3mm 0 3mm}, clip, width=0.98\textwidth]{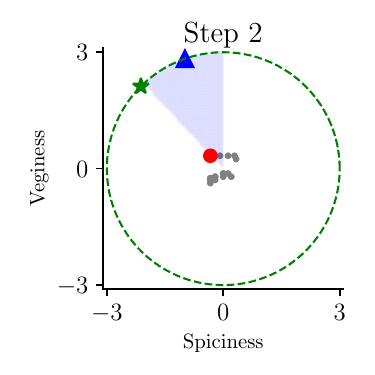}
            \end{tcolorbox}
        \end{minipage}%
        \hfill
        \begin{minipage}{0.67\textwidth}
            \begin{tcolorbox}[pairtext, fontupper=\footnotesize]
                How about a comforting vegetable risotto with fresh herbs like basil and parsley, where the creamy arborio rice perfectly complements the flavors of sautéed mushrooms, peas, and asparagus?
            \end{tcolorbox}
        \end{minipage}
    \end{tcolorbox}
    
    \caption{
    \looseness-1Illustrative example in the \texttt{foods2d} domain showing inference-time personalized alignment (see also \iftoggle{MainSuppContent}{Figures~\ref{fig.introduction}~and~\ref{fig.synth_2d})}{Figures~1~and~2)}. In Stage 3, $y_t^{(1)}$ appears above $y_t^{(2)}$ at each step (preferred {\color{green!60!black}{highlighted}}). The user is simulated by GPT-4o-mini, conditioned on the persona: ``\emph{A plant-based eater who avoids all heat and meat (spiciness: $-1.0$, veginess: $1.0$), preferring gentle, nourishing dishes with no spice.}''
    }
    \label{app-fig.illustration}
\end{figure*}

\subsection{Illustrative Example of Personalized Alignment}

Figure~\ref{app-fig.illustration} provides a concrete step-by-step illustration of inference-time personalized alignment in the \texttt{foods2d} domain. In addition to showing each stage of the interaction (i.e, the user question, candidate response generation, preference elicitation, and final response selection; also see \iftoggle{MainSuppContent}{Figure~\ref{fig.introduction}}{Figure~1}), the figure also visualizes how the version space is progressively reduced throughout the process (also see \iftoggle{MainSuppContent}{Figure~\ref{fig.synth_2d}}{Figure~2}). This integrated example highlights both the workflow of \algoOurs{} and the mechanism by which it narrows the set of plausible user preferences.

\subsection{Empirical Results with GPT-based Users}

Figure~\ref{app-fig.results_foods2d_gpt} shows results on \texttt{food2d} with GPT-based simulated users (also see \iftoggle{MainSuppContent}{Figure~\ref{fig.results_foods2d}}{Figure~3}). \algoOurs{} consistently achieves high win-rate at a low interaction cost, confirming that \algoOurs{} remains effective in the 2D setting even when user preferences are simulated by a large language model.


\begin{figure*}[h!]
  \centering

  \begin{subfigure}[t]{\textwidth}
    \centering
    \includegraphics[trim=0.5mm 2mm 0.5mm 2mm, clip, width=0.99\textwidth]{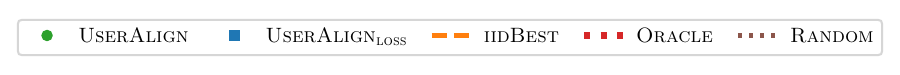}
  \end{subfigure}
  \vspace{-5mm}

  \begin{subfigure}[t]{0.55\textwidth}
    \centering
    \includegraphics[width=\textwidth]{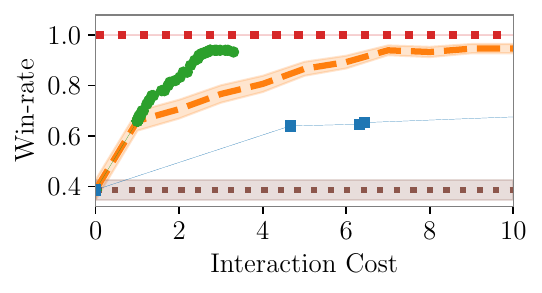}
    \vspace{-3mm}
    \caption{Win-rate vs. interaction cost}
    \label{app-fig.results_foods2d_gpt.cost}
  \end{subfigure}

  \vspace{2mm}

  \begin{subfigure}[t]{\textwidth}
    \centering
    \includegraphics[trim=0.5mm 2mm 0.5mm 2mm, clip, width=0.99\textwidth]{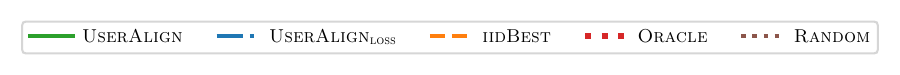}
  \end{subfigure}
  \vspace{-5mm}

  \begin{subfigure}[t]{0.55\textwidth}
    \centering
    \includegraphics[width=\textwidth]{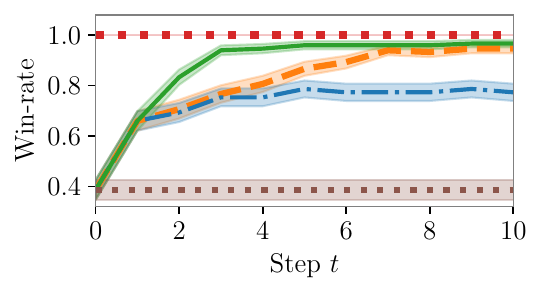}
    \vspace{-3mm}
    \caption{Win-rate w.r.t. increasing steps}
    \label{app-fig.results_foods2d_gpt.performance}
  \end{subfigure}

  \caption{Results on \texttt{food2d} for $u=\text{GPT}$. Top plot shows win-rate vs. interaction cost trade-off; here \algoOurs{} and \algoOursLoss{} results show scatter plot corresponding to varying values of $\epsilon$. Bottom plot shows win-rate per increasing interaction steps; here \algoOurs{} and \algoOursLoss{} are run for a given number of steps as mentioned in \iftoggle{MainSuppContent}{Footnote~\ref{footnote.algo_no_stopping}}{Footnote~1}. In these plots, \textsc{Oracle} and \textsc{Random} are flat lines, and \textsc{iidBest} results are reported at different number of given steps.}
  \label{app-fig.results_foods2d_gpt}
\end{figure*}

    \clearpage
\section{Results in Tabular Form}
\label{sec-app.tabular}

Tables~\ref{app-tab.tood_Travel}~and~\ref{app-tab.visual_dsp} offer additional details regarding the results, complementary to \iftoggle{MainSuppContent}{Figure~\ref{fig.four_domains_half}}{Figure~4}. The number of steps is capped at $199$.  Colored entries for \algoOurs{} and \algoOursLoss{} correspond to $\epsilon=0.0$, the default and most interpretable alignment setting. For \textsc{iidBest}, highlighted entries are chosen to match the interaction cost of the main \algoOurs{} configuration, enabling direct comparison. \textsc{Oracle} and \textsc{Random} are included as fixed reference baselines. \algoOurs{} consistently reaches high win-rates at low interaction cost, especially as the stopping threshold $\epsilon$ decreases, confirming its efficiency in identifying user-aligned responses with minimal queries. The reported standard errors are small, indicating reliable and stable performance.

Table~\ref{tab.head_to_head_all} complements these results with direct head-to-head win-rates of \algoOurs{} versus \textsc{iidBest}. These results confirm that the superior win-rates observed against a fixed baseline also extend to one-on-one comparisons across all domains.

%


\begin{table*}[h]
\caption{Win-rate vs. interaction cost trade-off on domains with arbitrary preference spaces (\texttt{food64d} and \texttt{travel64d} domains). Results are presented as mean (sem), complementing the results shown in \iftoggle{MainSuppContent}{Figure~\ref{fig.four_domains_half}}{Figure~4}. For readability, we report win-rates as percentages.}
\vspace{1mm}
\centering
\scalebox{0.93}{
\setlength\tabcolsep{6pt}
\renewcommand{\arraystretch}{1.2}
\begin{tabular}{l cc cc}
\toprule
\multirow{2}{*}{\textbf{Method}} & \multicolumn{2}{c}{\texttt{food64d}} & \multicolumn{2}{c}{\texttt{travel64d}} \\
\cmidrule(lr){2-3} \cmidrule(lr){4-5}
& Win-rate (\%) & Cost & Win-rate (\%) & Cost \\
\midrule
\algoOurs{} ($\epsilon=3.0$) & $ \phantom{0}54.00~(\phantom{0}4.07) $ & $ \phantom{0}\phantom{0}0.00~(\phantom{0}0.00) $ & $ \phantom{0}58.67~(\phantom{0}4.02) $ & $ \phantom{0}\phantom{0}0.00~(\phantom{0}0.00) $ \\
\algoOurs{} ($\epsilon=2.0$) & $ \phantom{0}74.67~(\phantom{0}3.55) $ & $ \phantom{0}\phantom{0}1.00~(\phantom{0}0.00) $ & $ \phantom{0}74.00~(\phantom{0}3.58) $ & $ \phantom{0}\phantom{0}1.00~(\phantom{0}0.00) $ \\
\algoOurs{} ($\epsilon=1.0$) & $ \phantom{0}84.67~(\phantom{0}2.94) $ & $ \phantom{0}\phantom{0}2.19~(\phantom{0}0.07) $ & $ \phantom{0}76.00~(\phantom{0}3.49) $ & $ \phantom{0}\phantom{0}1.95~(\phantom{0}0.07) $ \\
\algoOurs{} ($\epsilon=0.5$) & $ 100.00~(\phantom{0}0.00) $ & $ \phantom{0}11.39~(\phantom{0}0.23) $ & $ \phantom{0}94.67~(\phantom{0}1.83) $ & $ \phantom{0}11.57~(\phantom{0}0.23) $ \\
\algoOurs{} ($\epsilon=0.4$) & $ 100.00~(\phantom{0}0.00) $ & $ \phantom{0}13.47~(\phantom{0}0.22) $ & $ \phantom{0}97.33~(\phantom{0}1.32) $ & $ \phantom{0}13.82~(\phantom{0}0.22) $ \\
\algoOurs{} ($\epsilon=0.3$) & $ \phantom{0}98.67~(\phantom{0}0.94) $ & $ \phantom{0}15.49~(\phantom{0}0.25) $ & $ 100.00~(\phantom{0}0.00) $ & $ \phantom{0}16.00~(\phantom{0}0.23) $ \\
\algoOurs{} ($\epsilon=0.2$) & $ \phantom{0}97.33~(\phantom{0}1.32) $ & $ \phantom{0}16.85~(\phantom{0}0.29) $ & $ 100.00~(\phantom{0}0.00) $ & $ \phantom{0}18.31~(\phantom{0}0.29) $ \\
\algoOurs{} ($\epsilon=0.1$) & $ \phantom{0}98.67~(\phantom{0}0.94) $ & $ \phantom{0}18.28~(\phantom{0}0.30) $ & $ 100.00~(\phantom{0}0.00) $ & $ \phantom{0}19.67~(\phantom{0}0.23) $ \\
\rowcolor{green!20}
\algoOurs{} ($\epsilon=0.0$) & $ \phantom{0}98.67~(\phantom{0}0.94) $ & $ \phantom{0}19.13~(\phantom{0}0.24) $ & $ 100.00~(\phantom{0}0.00) $ & $ \phantom{0}20.69~(\phantom{0}0.21) $ \\
\midrule
\algoOursLoss{} ($\epsilon=3.0$) & $ \phantom{0}54.00~(\phantom{0}4.07) $ & $ \phantom{0}\phantom{0}0.00~(\phantom{0}0.00) $ & $ \phantom{0}58.67~(\phantom{0}4.02) $ & $ \phantom{0}\phantom{0}0.00~(\phantom{0}0.00) $ \\
\algoOursLoss{} ($\epsilon=2.0$) & $ \phantom{0}74.67~(\phantom{0}3.55) $ & $ \phantom{0}\phantom{0}1.00~(\phantom{0}0.00) $ & $ \phantom{0}74.00~(\phantom{0}3.58) $ & $ \phantom{0}\phantom{0}1.00~(\phantom{0}0.00) $ \\
\algoOursLoss{} ($\epsilon=1.0$) & $ \phantom{0}82.67~(\phantom{0}3.09) $ & $ 180.52~(\phantom{0}4.70) $ & $ \phantom{0}76.00~(\phantom{0}3.49) $ & $ 195.04~(\phantom{0}2.26) $ \\
\algoOursLoss{} ($\epsilon=0.5$) & $ \phantom{0}84.67~(\phantom{0}2.94) $ & $ 199.00~(\phantom{0}0.00) $ & $ \phantom{0}76.67~(\phantom{0}3.45) $ & $ 199.00~(\phantom{0}0.00) $ \\
\algoOursLoss{} ($\epsilon=0.4$) & $ \phantom{0}84.67~(\phantom{0}2.94) $ & $ 199.00~(\phantom{0}0.00) $ & $ \phantom{0}76.67~(\phantom{0}3.45) $ & $ 199.00~(\phantom{0}0.00) $ \\
\algoOursLoss{} ($\epsilon=0.3$) & $ \phantom{0}84.67~(\phantom{0}2.94) $ & $ 199.00~(\phantom{0}0.00) $ & $ \phantom{0}76.67~(\phantom{0}3.45) $ & $ 199.00~(\phantom{0}0.00) $ \\
\algoOursLoss{} ($\epsilon=0.2$) & $ \phantom{0}84.67~(\phantom{0}2.94) $ & $ 199.00~(\phantom{0}0.00) $ & $ \phantom{0}76.67~(\phantom{0}3.45) $ & $ 199.00~(\phantom{0}0.00) $ \\
\algoOursLoss{} ($\epsilon=0.1$) & $ \phantom{0}84.67~(\phantom{0}2.94) $ & $ 199.00~(\phantom{0}0.00) $ & $ \phantom{0}76.67~(\phantom{0}3.45) $ & $ 199.00~(\phantom{0}0.00) $ \\
\rowcolor{blue!20}
\algoOursLoss{} ($\epsilon=0.0$) & $ \phantom{0}84.67~(\phantom{0}2.94) $ & $ 199.00~(\phantom{0}0.00) $ & $ \phantom{0}76.67~(\phantom{0}3.45) $ & $ 199.00~(\phantom{0}0.00) $ \\
\midrule
\textsc{iidBest} ($t=0$) & $ \phantom{0}54.00~(\phantom{0}4.07) $ & $ \phantom{0}\phantom{0}0.00~(\phantom{0}0.00) $ & $ \phantom{0}58.67~(\phantom{0}4.02) $ & $ \phantom{0}\phantom{0}0.00~(\phantom{0}0.00) $ \\
\textsc{iidBest} ($t=5$) & $ \phantom{0}87.33~(\phantom{0}2.72) $ & $ \phantom{0}\phantom{0}5.00~(\phantom{0}0.00) $ & $ \phantom{0}78.00~(\phantom{0}3.38) $ & $ \phantom{0}\phantom{0}5.00~(\phantom{0}0.00) $ \\
\textsc{iidBest} ($t=10$) & $ \phantom{0}88.67~(\phantom{0}2.59) $ & $ \phantom{0}10.00~(\phantom{0}0.00) $ & $ \phantom{0}86.67~(\phantom{0}2.78) $ & $ \phantom{0}10.00~(\phantom{0}0.00) $ \\
\textsc{iidBest} ($t=15$) & $ \phantom{0}90.00~(\phantom{0}2.45) $ & $ \phantom{0}15.00~(\phantom{0}0.00) $ & $ \phantom{0}85.33~(\phantom{0}2.89) $ & $ \phantom{0}15.00~(\phantom{0}0.00) $ \\
\rowcolor{orange!20}
\textsc{iidBest} ($t=20$) & $ \phantom{0}90.00~(\phantom{0}2.45) $ & $ \phantom{0}20.00~(\phantom{0}0.00) $ & $ \phantom{0}85.33~(\phantom{0}2.89) $ & $ \phantom{0}20.00~(\phantom{0}0.00) $ \\
\rowcolor{orange!20}
\textsc{iidBest} ($t=25$) & $ \phantom{0}90.00~(\phantom{0}2.45) $ & $ \phantom{0}25.00~(\phantom{0}0.00) $ & $ \phantom{0}87.33~(\phantom{0}2.72) $ & $ \phantom{0}25.00~(\phantom{0}0.00) $ \\
\rowcolor{orange!20}
\textsc{iidBest} ($t=50$) & $ \phantom{0}92.00~(\phantom{0}2.22) $ & $ \phantom{0}50.00~(\phantom{0}0.00) $ & $ \phantom{0}88.00~(\phantom{0}2.65) $ & $ \phantom{0}50.00~(\phantom{0}0.00) $ \\
\textsc{iidBest} ($t=75$) & $ \phantom{0}92.67~(\phantom{0}2.13) $ & $ \phantom{0}75.00~(\phantom{0}0.00) $ & $ \phantom{0}90.00~(\phantom{0}2.45) $ & $ \phantom{0}75.00~(\phantom{0}0.00) $ \\
\textsc{iidBest} ($t=100$) & $ \phantom{0}92.67~(\phantom{0}2.13) $ & $ 100.00~(\phantom{0}0.00) $ & $ \phantom{0}90.00~(\phantom{0}2.45) $ & $ 100.00~(\phantom{0}0.00) $ \\
\textsc{iidBest} ($t=150$) & $ \phantom{0}93.33~(\phantom{0}2.04) $ & $ 150.00~(\phantom{0}0.00) $ & $ \phantom{0}91.33~(\phantom{0}2.30) $ & $ 150.00~(\phantom{0}0.00) $ \\
\textsc{iidBest} ($t=199$) & $ \phantom{0}93.33~(\phantom{0}2.04) $ & $ 199.00~(\phantom{0}0.00) $ & $ \phantom{0}90.67~(\phantom{0}2.38) $ & $ 199.00~(\phantom{0}0.00) $ \\
\midrule
\rowcolor{red!20}
\textsc{Oracle} & $ 100.00~(\phantom{0}0.00) $ & $ \phantom{0}\phantom{0}0.00~(\phantom{0}0.00) $ & $ 100.00~(\phantom{0}0.00) $ & $ \phantom{0}\phantom{0}0.00~(\phantom{0}0.00) $ \\
\midrule
\rowcolor{brown!20}
\textsc{Random} & $ \phantom{0}54.00~(\phantom{0}4.07) $ & $ \phantom{0}\phantom{0}0.00~(\phantom{0}0.00) $ & $ \phantom{0}58.67~(\phantom{0}4.02) $ & $ \phantom{0}\phantom{0}0.00~(\phantom{0}0.00) $ \\
\bottomrule
\end{tabular}
}
\label{app-tab.tood_Travel}
\end{table*}

\begin{table*}[h]
\caption{Win-rate vs. interaction cost trade-off on domains with arbitrary preference spaces (\texttt{visual512d} and \texttt{dsp64d} domains). Results are presented as mean (sem), complementing the results shown in \iftoggle{MainSuppContent}{Figure~\ref{fig.four_domains_half}}{Figure~4}. For readability, we report win-rates as percentages.}
\vspace{1mm}
\centering
\scalebox{0.93}{
\setlength\tabcolsep{6pt}
\renewcommand{\arraystretch}{1.2}
\begin{tabular}{l cc cc}
\toprule
\multirow{2}{*}{\textbf{Method}} & \multicolumn{2}{c}{\texttt{visual512d}} & \multicolumn{2}{c}{\texttt{dsp64d}} \\
\cmidrule(lr){2-3} \cmidrule(lr){4-5}
& Win-rate (\%) & Cost & Win-rate (\%) & Cost \\
\midrule
\algoOurs{} ($\epsilon=3.0$) & $ \phantom{0}40.67~(\phantom{0}4.01) $ & $ \phantom{0}\phantom{0}0.00~(\phantom{0}0.00) $ & $ \phantom{0}51.73~(\phantom{0}1.29) $ & $ \phantom{0}\phantom{0}0.00~(\phantom{0}0.00) $ \\
\algoOurs{} ($\epsilon=2.0$) & $ \phantom{0}58.00~(\phantom{0}4.03) $ & $ \phantom{0}\phantom{0}1.00~(\phantom{0}0.00) $ & $ \phantom{0}65.33~(\phantom{0}1.23) $ & $ \phantom{0}\phantom{0}1.00~(\phantom{0}0.00) $ \\
\algoOurs{} ($\epsilon=1.0$) & $ \phantom{0}65.33~(\phantom{0}3.89) $ & $ \phantom{0}\phantom{0}7.58~(\phantom{0}0.25) $ & $ \phantom{0}70.73~(\phantom{0}1.17) $ & $ \phantom{0}\phantom{0}1.49~(\phantom{0}0.02) $ \\
\algoOurs{} ($\epsilon=0.5$) & $ \phantom{0}86.00~(\phantom{0}2.83) $ & $ \phantom{0}34.09~(\phantom{0}0.68) $ & $ \phantom{0}84.80~(\phantom{0}0.93) $ & $ \phantom{0}\phantom{0}5.03~(\phantom{0}0.07) $ \\
\algoOurs{} ($\epsilon=0.4$) & $ \phantom{0}92.00~(\phantom{0}2.22) $ & $ \phantom{0}41.12~(\phantom{0}0.74) $ & $ \phantom{0}88.53~(\phantom{0}0.82) $ & $ \phantom{0}\phantom{0}7.53~(\phantom{0}0.09) $ \\
\algoOurs{} ($\epsilon=0.3$) & $ \phantom{0}93.33~(\phantom{0}2.04) $ & $ \phantom{0}47.57~(\phantom{0}0.71) $ & $ \phantom{0}91.33~(\phantom{0}0.73) $ & $ \phantom{0}11.53~(\phantom{0}0.11) $ \\
\algoOurs{} ($\epsilon=0.2$) & $ \phantom{0}94.67~(\phantom{0}1.83) $ & $ \phantom{0}49.31~(\phantom{0}0.72) $ & $ \phantom{0}94.67~(\phantom{0}0.58) $ & $ \phantom{0}17.07~(\phantom{0}0.13) $ \\
\algoOurs{} ($\epsilon=0.1$) & $ \phantom{0}94.67~(\phantom{0}1.83) $ & $ \phantom{0}49.31~(\phantom{0}0.72) $ & $ \phantom{0}96.00~(\phantom{0}0.51) $ & $ \phantom{0}20.58~(\phantom{0}0.14) $ \\
\rowcolor{green!20}
\algoOurs{} ($\epsilon=0.0$) & $ \phantom{0}94.67~(\phantom{0}1.83) $ & $ \phantom{0}49.31~(\phantom{0}0.72) $ & $ \phantom{0}96.27~(\phantom{0}0.49) $ & $ \phantom{0}21.12~(\phantom{0}0.14) $ \\
\midrule
\algoOursLoss{} ($\epsilon=3.0$) & $ \phantom{0}40.67~(\phantom{0}4.01) $ & $ \phantom{0}\phantom{0}0.00~(\phantom{0}0.00) $ & $ \phantom{0}51.73~(\phantom{0}1.29) $ & $ \phantom{0}\phantom{0}0.00~(\phantom{0}0.00) $ \\
\algoOursLoss{} ($\epsilon=2.0$) & $ \phantom{0}58.00~(\phantom{0}4.03) $ & $ \phantom{0}\phantom{0}1.00~(\phantom{0}0.00) $ & $ \phantom{0}65.33~(\phantom{0}1.23) $ & $ \phantom{0}\phantom{0}1.00~(\phantom{0}0.00) $ \\
\algoOursLoss{} ($\epsilon=1.0$) & $ \phantom{0}59.33~(\phantom{0}4.01) $ & $ 199.00~(\phantom{0}0.00) $ & $ \phantom{0}71.20~(\phantom{0}1.17) $ & $ 164.82~(\phantom{0}1.93) $ \\
\algoOursLoss{} ($\epsilon=0.5$) & $ \phantom{0}59.33~(\phantom{0}4.01) $ & $ 199.00~(\phantom{0}0.00) $ & $ \phantom{0}72.47~(\phantom{0}1.15) $ & $ 199.00~(\phantom{0}0.00) $ \\
\algoOursLoss{} ($\epsilon=0.4$) & $ \phantom{0}59.33~(\phantom{0}4.01) $ & $ 199.00~(\phantom{0}0.00) $ & $ \phantom{0}72.47~(\phantom{0}1.15) $ & $ 199.00~(\phantom{0}0.00) $ \\
\algoOursLoss{} ($\epsilon=0.3$) & $ \phantom{0}59.33~(\phantom{0}4.01) $ & $ 199.00~(\phantom{0}0.00) $ & $ \phantom{0}72.47~(\phantom{0}1.15) $ & $ 199.00~(\phantom{0}0.00) $ \\
\algoOursLoss{} ($\epsilon=0.2$) & $ \phantom{0}59.33~(\phantom{0}4.01) $ & $ 199.00~(\phantom{0}0.00) $ & $ \phantom{0}72.47~(\phantom{0}1.15) $ & $ 199.00~(\phantom{0}0.00) $ \\
\algoOursLoss{} ($\epsilon=0.1$) & $ \phantom{0}59.33~(\phantom{0}4.01) $ & $ 199.00~(\phantom{0}0.00) $ & $ \phantom{0}72.47~(\phantom{0}1.15) $ & $ 199.00~(\phantom{0}0.00) $ \\
\rowcolor{blue!20}
\algoOursLoss{} ($\epsilon=0.0$) & $ \phantom{0}59.33~(\phantom{0}4.01) $ & $ 199.00~(\phantom{0}0.00) $ & $ \phantom{0}72.47~(\phantom{0}1.15) $ & $ 199.00~(\phantom{0}0.00) $ \\
\midrule
\textsc{iidBest} ($t=0$) & $ \phantom{0}40.67~(\phantom{0}4.01) $ & $ \phantom{0}\phantom{0}0.00~(\phantom{0}0.00) $ & $ \phantom{0}51.73~(\phantom{0}1.29) $ & $ \phantom{0}\phantom{0}0.00~(\phantom{0}0.00) $ \\
\textsc{iidBest} ($t=5$) & $ \phantom{0}65.33~(\phantom{0}3.89) $ & $ \phantom{0}\phantom{0}5.00~(\phantom{0}0.00) $ & $ \phantom{0}75.33~(\phantom{0}1.11) $ & $ \phantom{0}\phantom{0}5.00~(\phantom{0}0.00) $ \\
\textsc{iidBest} ($t=10$) & $ \phantom{0}68.00~(\phantom{0}3.81) $ & $ \phantom{0}10.00~(\phantom{0}0.00) $ & $ \phantom{0}77.87~(\phantom{0}1.07) $ & $ \phantom{0}10.00~(\phantom{0}0.00) $ \\
\textsc{iidBest} ($t=15$) & $ \phantom{0}69.33~(\phantom{0}3.76) $ & $ \phantom{0}15.00~(\phantom{0}0.00) $ & $ \phantom{0}79.73~(\phantom{0}1.04) $ & $ \phantom{0}15.00~(\phantom{0}0.00) $ \\
\rowcolor{orange!20}
\textsc{iidBest} ($t=20$) & $ \phantom{0}77.33~(\phantom{0}3.42) $ & $ \phantom{0}20.00~(\phantom{0}0.00) $ & $ \phantom{0}81.60~(\phantom{0}1.00) $ & $ \phantom{0}20.00~(\phantom{0}0.00) $ \\
\rowcolor{orange!20}
\textsc{iidBest} ($t=25$) & $ \phantom{0}74.67~(\phantom{0}3.55) $ & $ \phantom{0}25.00~(\phantom{0}0.00) $ & $ \phantom{0}81.67~(\phantom{0}1.00) $ & $ \phantom{0}25.00~(\phantom{0}0.00) $ \\
\rowcolor{orange!20}
\textsc{iidBest} ($t=50$) & $ \phantom{0}84.00~(\phantom{0}2.99) $ & $ \phantom{0}50.00~(\phantom{0}0.00) $ & $ \phantom{0}82.47~(\phantom{0}0.98) $ & $ \phantom{0}50.00~(\phantom{0}0.00) $ \\
\textsc{iidBest} ($t=75$) & $ \phantom{0}83.33~(\phantom{0}3.04) $ & $ \phantom{0}75.00~(\phantom{0}0.00) $ & $ \phantom{0}83.60~(\phantom{0}0.96) $ & $ \phantom{0}75.00~(\phantom{0}0.00) $ \\
\textsc{iidBest} ($t=100$) & $ \phantom{0}86.67~(\phantom{0}2.78) $ & $ 100.00~(\phantom{0}0.00) $ & $ \phantom{0}84.27~(\phantom{0}0.94) $ & $ 100.00~(\phantom{0}0.00) $ \\
\textsc{iidBest} ($t=150$) & $ \phantom{0}86.00~(\phantom{0}2.83) $ & $ 150.00~(\phantom{0}0.00) $ & $ \phantom{0}84.87~(\phantom{0}0.93) $ & $ 150.00~(\phantom{0}0.00) $ \\
\textsc{iidBest} ($t=199$) & $ \phantom{0}84.67~(\phantom{0}2.94) $ & $ 199.00~(\phantom{0}0.00) $ & $ \phantom{0}84.47~(\phantom{0}0.94) $ & $ 199.00~(\phantom{0}0.00) $ \\
\midrule
\rowcolor{red!20}
\textsc{Oracle} & $ 100.00~(\phantom{0}0.00) $ & $ \phantom{0}\phantom{0}0.00~(\phantom{0}0.00) $ & $ \phantom{0}99.33~(\phantom{0}0.21) $ & $ \phantom{0}\phantom{0}0.00~(\phantom{0}0.00) $ \\
\midrule
\rowcolor{brown!20}
\textsc{Random} & $ \phantom{0}40.67~(\phantom{0}4.01) $ & $ \phantom{0}\phantom{0}0.00~(\phantom{0}0.00) $ & $ \phantom{0}51.73~(\phantom{0}1.29) $ & $ \phantom{0}\phantom{0}0.00~(\phantom{0}0.00) $ \\
\bottomrule
\end{tabular}
}
\label{app-tab.visual_dsp}
\end{table*}


%

\begin{table*}[h]
\caption{Head-to-head win-rate (\%) comparison between \algoOurs{} ($\epsilon=0.0$) and \textsc{iidBest} across four domains. Results are presented as mean (sem).}
\vspace{1mm}
\centering
\scalebox{0.93}{
\setlength\tabcolsep{8pt}
\renewcommand{\arraystretch}{1.2}
\begin{tabular}{lcccc}
\toprule
\multirow{2}{*}{\textbf{Comparison}} & \multicolumn{4}{c}{Win-rate (\%)} \\
\cmidrule(lr){2-5}
& \texttt{food64d} & \texttt{travel64d} & \texttt{visual512d} & \texttt{dsp64d} \\
\midrule
\algoOurs{} vs. \textsc{iidBest} ($t=20$) & $96.67~(1.47)$ & $93.33~(2.04)$ & $85.33~(2.89)$ & $91.00~(0.74)$ \\
\algoOurs{} vs. \textsc{iidBest} ($t=25$) & $97.33~(1.32)$ & $95.33~(1.72)$ & $85.33~(2.89)$ & $90.33~(0.76)$ \\
\algoOurs{} vs. \textsc{iidBest} ($t=50$) & $96.67~(1.47)$ & $95.33~(1.72)$ & $87.33~(2.72)$ & $90.07~(0.77)$ \\
\bottomrule
\end{tabular}
}
\label{tab.head_to_head_all}
\end{table*}

    \clearpage
\section{Pool Generation Details and Additional Experiments}
\label{sec-app.experiments}

\subsection{Details about Pool Generation Used in \iftoggle{MainSuppContent}{Section~\ref{sec.experiments}}{Section~5}}

Below we provide details about the pool generation that was used for evaluation in \iftoggle{MainSuppContent}{Section~\ref{sec.experiments}}{Section~5}.

\textbf{Pool generation procedure for \texttt{food64d}, \texttt{travel64d}, and \texttt{visual512d}.} As introduced in \iftoggle{MainSuppContent}{Section~\ref{sec.experiments.setup}}{Section~5.1}, for each question, we generate a diverse pool of candidate responses by following two sampling approaches.
\begin{itemize}
    \item We obtain half of the candidate responses by sampling from GPT-4o or GPT-Image-1 temperature $0.5$, conditioned only on the question, resulting in unbiased responses. We generate $10$ responses per question in \texttt{food64d} and \texttt{travel64d}, and $20$ in \texttt{visual512d}.
    \item We obtain the remaining half by prompting the generative model to first reason about possible diverse interests relevant to the question, at temperature $0.8$. For each generated interest, we then prompt GPT-4o (for text) or GPT-Image-1 (for vision) at temperature $0.5$ to produce a candidate response. We generate $10$ interests per question in \texttt{food64d} and \texttt{travel64d}, and $20$ in \texttt{visual512d}, with one response generated per interest.
\end{itemize}

\textbf{Pool generation procedure for \texttt{dsp64d}.}
Similar to the procedure discussed above, here we also generate a diverse pool of candidate responses by following slightly different approaches.
\begin{itemize}
    \item We obtain half of the candidate responses by sampling from GPT-4o at temperature $0.5$, conditioned only on the question, resulting in unbiased responses. We generate $10$ responses per question.
    \item We obtain the remaining half by using the interests provided in the DSP dataset. We generate $5$ responses for each interest using GPT-4o at temperature $0.5$, and from this collection, we sample $10$ responses to obtain the remaining half.
\end{itemize}

\subsection{Additional Experiments with Different Candidate Pools}

Below we provide evaluation results for experiments using different candidate pool to further assess the utility and robustness of our method.

\textbf{Unbiased response pool construction.}
For each question, we generate a pool of candidate responses by following one sampling approach. More concretely, responses are sampled from GPT-4o or GPT-Image-1 (at temperature $0.5$), conditioned only on the question. The pool size remains the same as in the main experiments, i.e., $20$ for text domains and $40$ for the image-based domain.

\textbf{Results.}
Figure~\ref{app-fig.results_base_pool} summarizes the results with these pools of unbiased responses. \algoOurs{} continues to achieve high win-rate at low interaction cost. These results demonstrate that \algoOurs{} remains effective and outperform baselines also when applied to this variant of candidate pools.


\begin{figure*}[h!]
  \centering

    \begin{subfigure}[t]{\textwidth}
      \centering
      \includegraphics[trim=0.5mm 2mm 0.5mm 2mm, clip, width=0.82\textwidth]{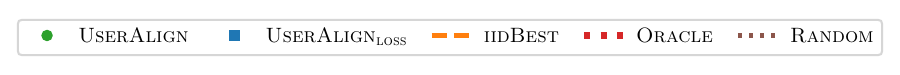}
    \end{subfigure}

    \begin{subfigure}[t]{0.48\textwidth}
      \includegraphics[height=3cm]{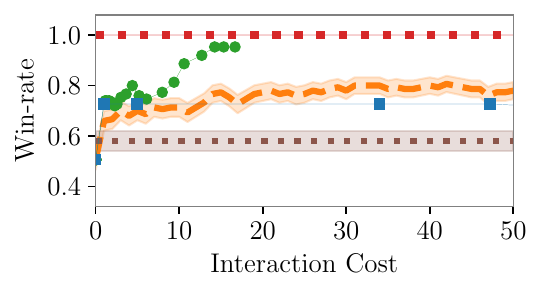}
      \vspace{-2.5mm}
      \caption{Domain \texttt{food64d}}
    \end{subfigure}
    \ \ \ \  
    \begin{subfigure}[t]{0.48\textwidth}
      \includegraphics[height=3cm]{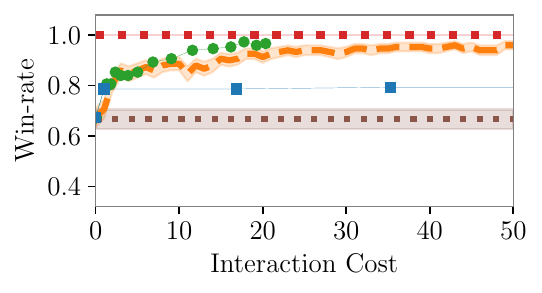}
      \vspace{-2.5mm}
      \subcaption{Domain \texttt{travel64d}}
    \end{subfigure}
    
    \begin{subfigure}[t]{0.48\textwidth}
      \includegraphics[height=3cm]{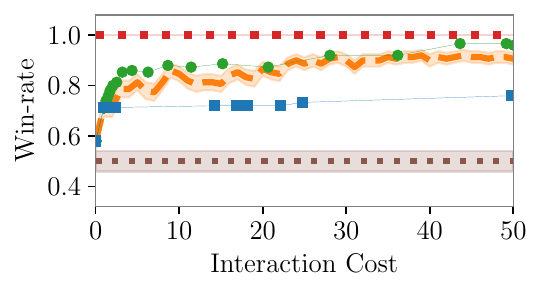}
      \vspace{-2.5mm}
      \subcaption{Domain \texttt{visual512d}}
    \end{subfigure}
    \ \ \ \  
    \begin{subfigure}[t]{0.48\textwidth}
      \includegraphics[height=3cm]{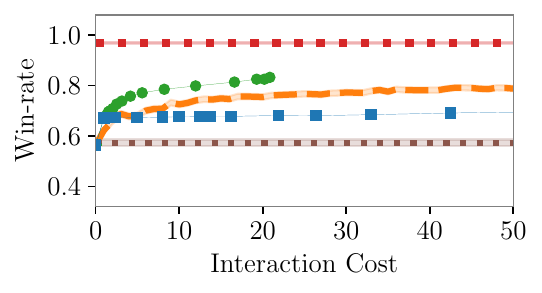}
      \vspace{-2.5mm}
      \caption{Domain \texttt{dsp64d}}
  \end{subfigure}    

  \caption{Win-rate vs. interaction cost trade-off on domains with arbitrary preference spaces with fully unbiased pools. Here \algoOurs{} and \algoOursLoss{} results show scatter plot corresponding to varying values of $\epsilon$. \textsc{Oracle} and \textsc{Random} are flat lines, and \textsc{iidBest} results are reported at different number of given steps.
  }
  \label{app-fig.results_base_pool}
\end{figure*}


\subsection{Practicality of Pairwise Comparisons vs. Best-of-N Selection}

As Best-of-$N$ requires users to scan a potentially large pool and pick a single winner, this approach becomes impractical as $N$ grows. Pairwise comparisons reduce this burden by focusing on one local decision at a time. Our approach relies solely on pairwise feedback, avoids repeating pairs, and maintains an incumbent that is continually challenged, encouraging early exploration and later convergence.

To demonstrate this, we experimented with increasing $N$. Table~\ref{app-tab.bon-vs-pairwise} shows the performance of \algoOurs{} at $N{=}1000$, where a usable Best-of-$N$ interface is no longer realistic. \algoOurs{} maintains high win-rates, demonstrating scalability without increasing user effort.

\begin{table}[h]
    \caption{\looseness-1Win-rate (\%, sem) on \texttt{food64d} across interaction budgets $t$ for a large pool size $N=1000$.}
    \vspace{5mm}
    \centering
    \scalebox{0.93}{
    \setlength\tabcolsep{12pt}
    \renewcommand{\arraystretch}{1.2}
    \begin{tabular}{lcccc
                    }
        \toprule
        \textbf{Method} & $t=0$ & $t=5$ & $t=10$ & $t=20$ \\
        \midrule
        \textsc{Random}     & $ \phantom{0}52.00~(4.09) $ & $ \phantom{0}52.00~(4.09) $ & $ \phantom{0}52.00~(4.09) $ & $ \phantom{0}52.00~(4.09) $ \\
        \textsc{iidBest}    & $ \phantom{0}52.00~(4.09) $ & $ \phantom{0}72.00~(3.68) $ & $ \phantom{0}76.00~(3.50) $ & $ \phantom{0}82.67~(3.10) $ \\
        \algoOurs{}          & $ \phantom{0}52.00~(4.09) $ & $ \phantom{0}86.67~(2.78) $ & $ \phantom{0}90.67~(2.38) $ & $ \phantom{0}96.67~(1.47) $ \\
        \bottomrule
    \end{tabular}}
    \label{app-tab.bon-vs-pairwise}
\end{table}

    \clearpage
    \section{Additional Details for Evaluation with Human Users}
\label{sec-app:userstudy}

\looseness-1Figures~\ref{app_fig.userstudy_text} and \ref{app_fig.userstudy_img} expand on the web application interface described in Section~\ref{sec.userstudy}. Figure~\ref{app_fig.userstudy_text} illustrates the \texttt{food64d} text-based workflow, showing the Stage~1 onboarding (for both personalization conditions), the Stage~2 comparisons, and the Stage~3 evaluation against baseline. Figure~\ref{app_fig.userstudy_img} presents the corresponding \texttt{visual512d} image-based workflow.

\begin{figure*}[ht!]
    \centering
    \fboxsep=0.5pt
    \fboxrule=0.5pt

    \begin{subfigure}{\textwidth}
        \centering
        \fcolorbox{gray}{white}{\includegraphics[width=\textwidth]{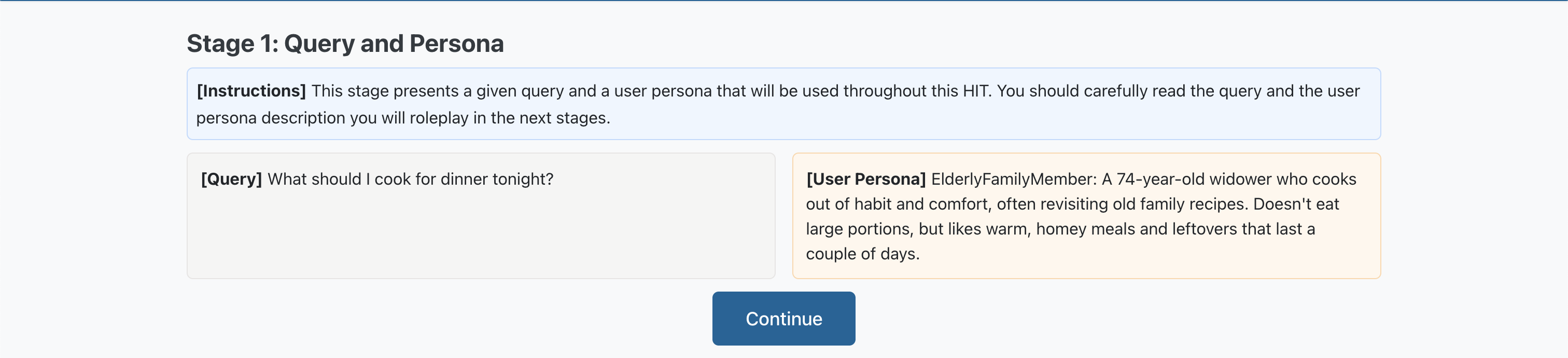}}
        \caption{\looseness-1Stage~1 for the \texttt{food64d} domain and with-persona condition.}
        \label{app_fig.userstudy_txt:a}
    \end{subfigure}

    \vspace{2mm}

    \begin{subfigure}{\textwidth}
        \centering
        \fcolorbox{gray}{white}{\includegraphics[width=\textwidth]{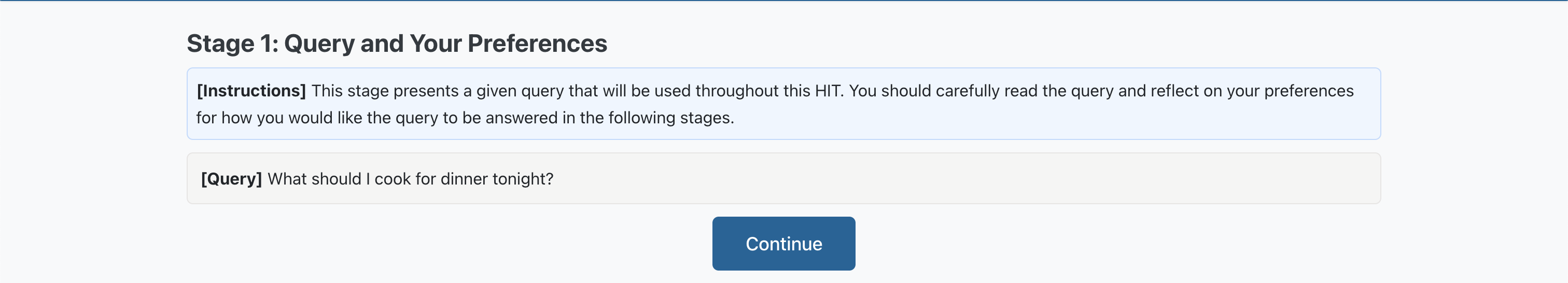}}
        \caption{\looseness-1Stage~1 for the \texttt{food64d} domain and  without-persona condition.}
        \label{app_fig.userstudy_txt:b}
    \end{subfigure}

    \vspace{2mm}

    \begin{subfigure}{\textwidth}
        \centering
        \fcolorbox{gray}{white}{\includegraphics[width=\textwidth]{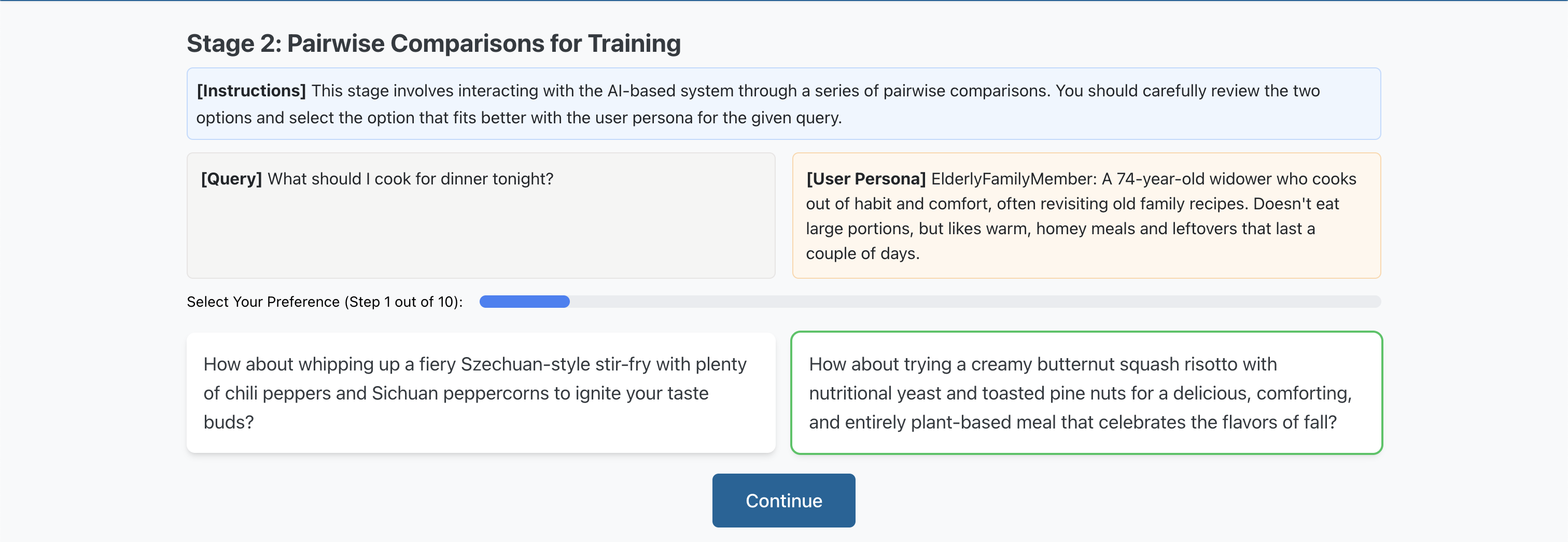}}
        \caption{Stage~2 for the \texttt{food64d} domain and with-persona condition.}
        \label{app_fig.userstudy_txt:c}
    \end{subfigure}

    \vspace{2mm}

    \begin{subfigure}{\textwidth}
        \centering
        \fcolorbox{gray}{white}{\includegraphics[width=\textwidth]{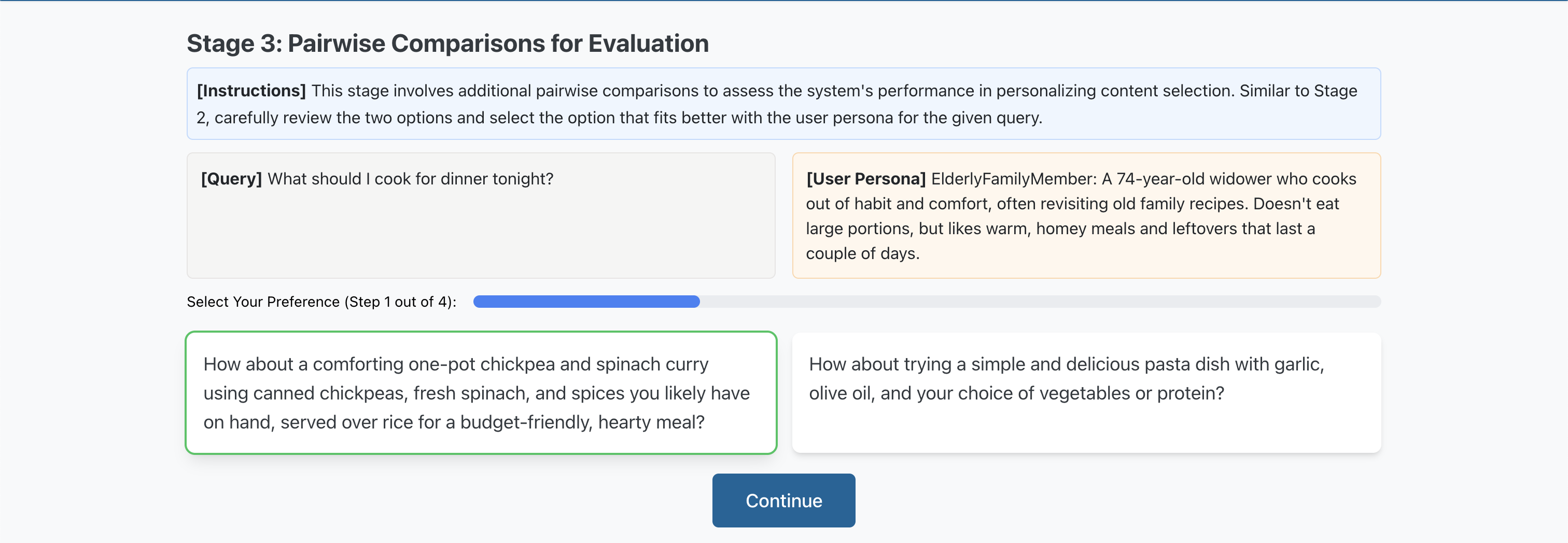}}
        \caption{Stage~3 for the \texttt{food64d} domain and with-persona condition.}
        \label{app_fig.userstudy_txt:d}
    \end{subfigure}

    \caption{Screenshots from the web application for the \texttt{food64d} text-based workflow.}
    \label{app_fig.userstudy_text}
\end{figure*}

\begin{figure*}[ht!]
    \centering
    \fboxsep=0.5pt
    \fboxrule=0.5pt

    \begin{subfigure}{\textwidth}
        \centering
        \fcolorbox{gray}{white}{\includegraphics[width=\textwidth]{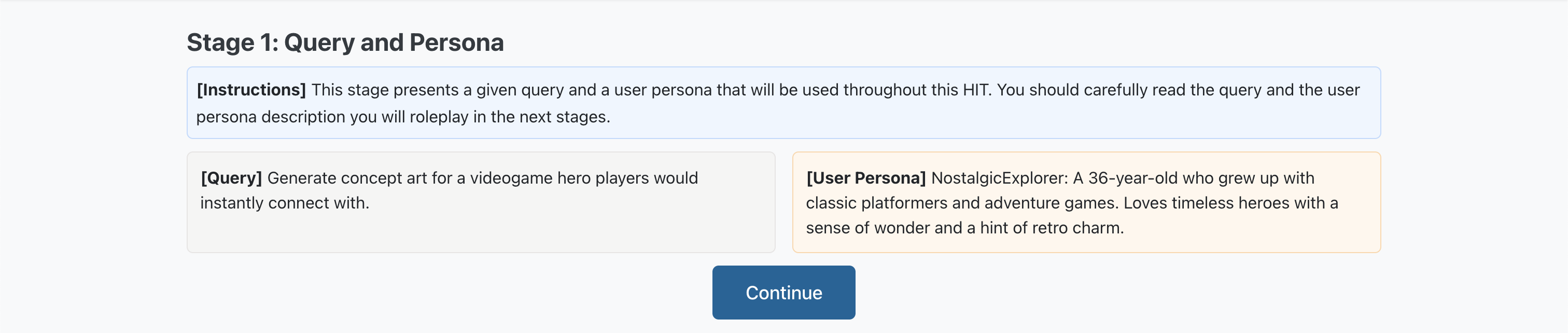}}
        \caption{\looseness-1Stage~1 for the \texttt{visual512d} domain and with-persona condition.}
        \label{app_fig.userstudy_img:a}
    \end{subfigure}

    \vspace{2mm}

    \begin{subfigure}{\textwidth}
        \centering
        \fcolorbox{gray}{white}{\includegraphics[width=\textwidth]{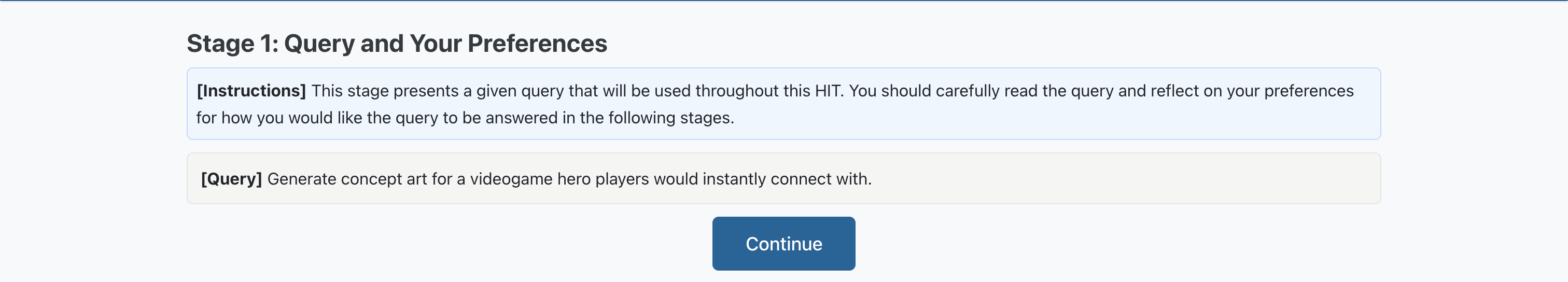}}
        \caption{\looseness-1Stage~1 for the \texttt{visual512d} domain and without-persona condition.}
        \label{app_fig.userstudy_img:b}
    \end{subfigure}

    \vspace{2mm}

    \begin{subfigure}{\textwidth}
        \centering
        \fcolorbox{gray}{white}{\includegraphics[width=\textwidth]{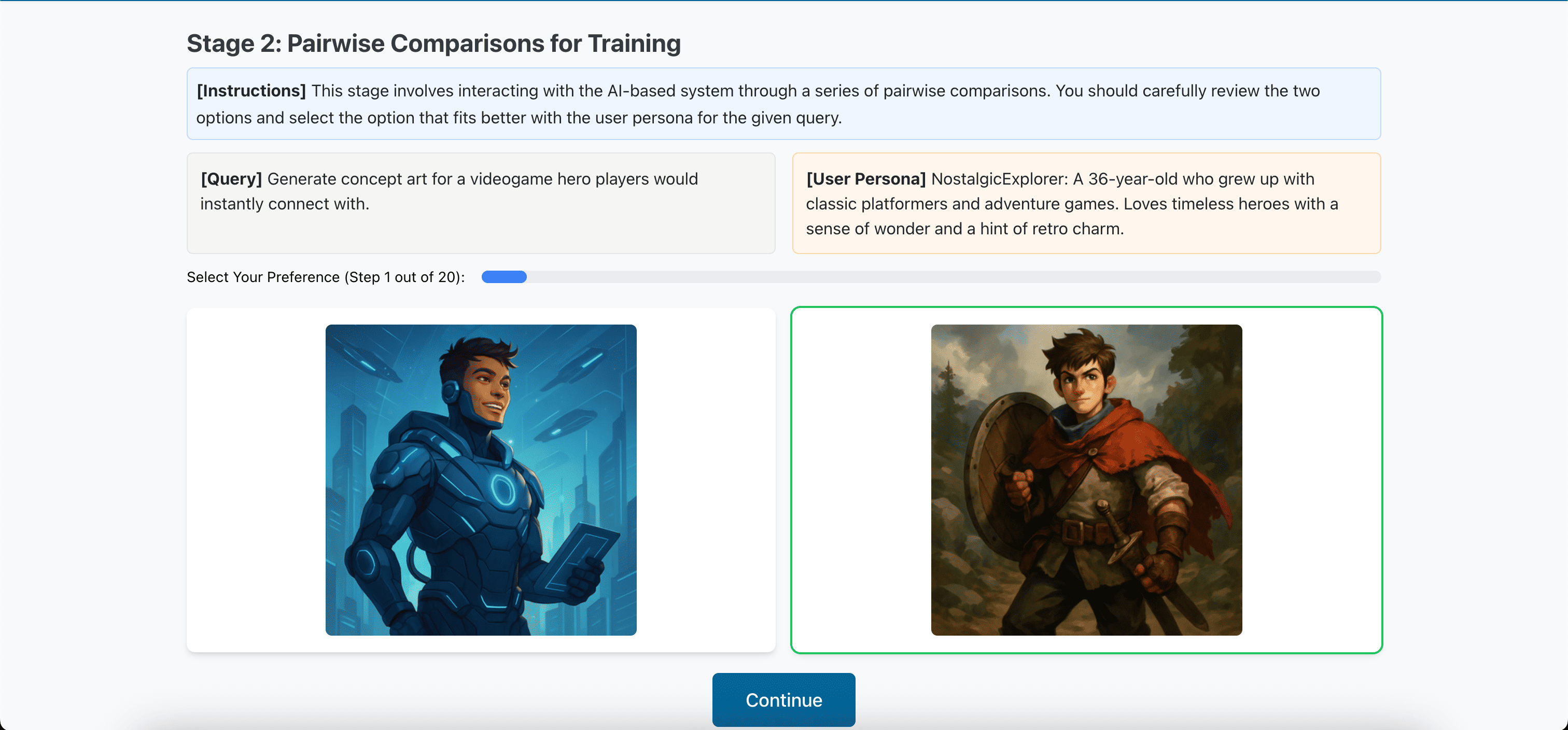}}
        \caption{Stage~2 for the \texttt{visual512d} domain and with-persona condition.}
        \label{app_fig.userstudy_img:c}
    \end{subfigure}

    \vspace{2mm}

    \begin{subfigure}{\textwidth}
        \centering
        \fcolorbox{gray}{white}{\includegraphics[width=\textwidth]{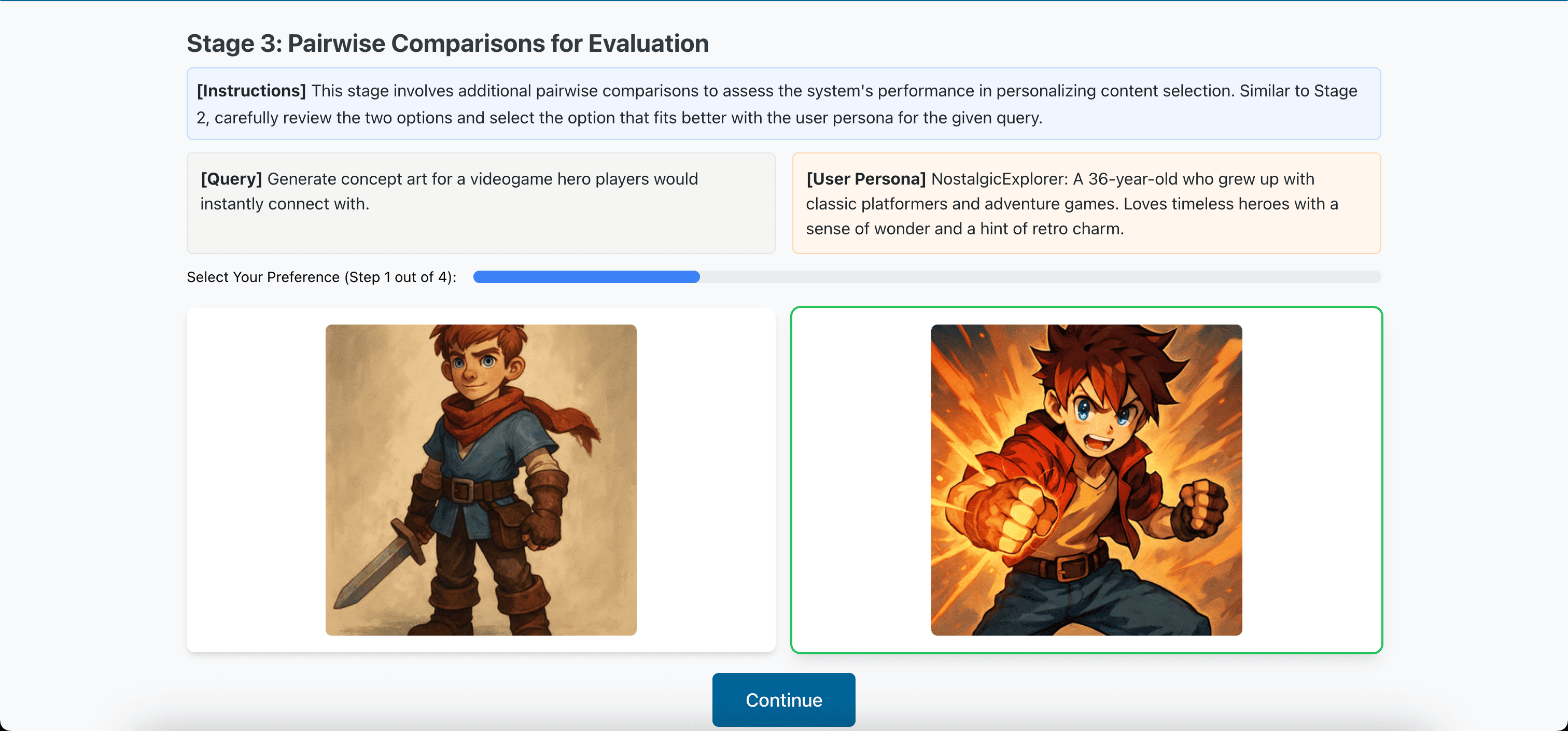}}
        \caption{Stage~3 for the \texttt{visual512d} domain and with-persona condition.}
        \label{app_fig.userstudy_img:d}
    \end{subfigure}

    \caption{\looseness-1Screenshots from the web application for the \texttt{visual512d} image-based workflow.}
    \label{app_fig.userstudy_img}
\end{figure*}

\clearpage

\subsection{Flow and On-Screen Instructions (Plain Text)}

\textbf{Landing.}
Participants first see an overview of the study and the system. The screen shows: ``Here, you will interact with an AI-based system. The system aims to provide personalized content for a given query and a user persona. The system will infer user preferences through pairwise comparisons while interacting with you. We are conducting this study as part of a research project''.

\textbf{Stage 1.}
First, participants see the query and the personalization condition. The screen shows the instructions, the query and, in the with-persona condition, also the user persona card; in the without-persona condition, it shows only the query. In the with-persona condition, the header shows ``Query and Persona'', and the instructions read: ``This stage presents a given query and a user persona that will be used throughout this HIT. Please read the query and the persona you will roleplay in the next stages''. In the without-persona condition, the header shows ``Query and Your Preferences'', and the instructions read: ``This stage presents a given query that will be used throughout this HIT. Please read the query and reflect on your preferences for how you would like the query to be answered in the following stages''.

\textbf{Stage 2.}
Next, participants perform pairwise comparisons. On top, the screen shows the instructions, the query and, in the with-persona condition, also the user persona card; in the without-persona condition, it shows only the query. The header shows ``Pairwise Comparisons for Training''. In the with-persona condition, the instructions read: ``Review the two options and select the one that better fits the user persona for the given query''. In the without-persona condition, the instructions read: ``Review the two options and select the one that best fits your own preferences for the given query''.

\textbf{Stage 3.}
Finally, participants complete additional pairwise comparisons for evaluation. On top, the screen shows the instructions, the query and, in the with-persona condition, also the user persona card; in the without-persona condition, it shows only the query. The header shows ``Pairwise Comparisons for Evaluation''. In this stage, participants compare, one matchup at a time, the system’s selected candidates at the two chosen interaction steps and a \textsc{Random} candidate against the shared zero-temperature baseline. In the with-persona condition, the instructions read: ``Review the two options and select the one that better fits the user persona for the given query''. In the without-persona condition, the instructions read: ``Review the two options and select the one that best aligns with your own preferences for the given query''.

    \clearpage
\section{Wall-clock, Compute, and API Costs}
\label{sec-app.costs}

All experiments ran on a compute node with dual AMD EPYC 7702 64-core processors (128 cores total) and 2TB DDR4 ECC memory (2933MHz).

\subsection{Computational Efficiency of \algoOurs{}}
\label{sec-app.comp-eff}

\looseness-1The optimization problem of computing $\widehat{\theta}_t$ in \iftoggle{MainSuppContent}{Eq.~\eqref{eq:mle-estimate}}{Eq.~2} is convex, with a convex objective $\mathcal{L}_t(\cdot)$ and a convex constraint set $\{\theta \in \mathbb{R}^d: \norm{\theta}_2 \leq S\}$. Moreover, the confidence set $\Theta_t$ defined in \iftoggle{MainSuppContent}{Eq.~\eqref{eq:confidence-set}}{Eq.~3} is also convex. The first response $y_t^{(1)} \gets \argmax_{y \in \mathcal{Y}_\text{cand}} \langle \widehat{\theta}_t, \phi(x, y) \rangle$ can be computed via $\abs{\mathcal{Y}_\text{cand}}$ inner product evaluations. The second response $(y_t^{(2)}, \widetilde{\theta}_t) \gets \argmax_{(y', \theta) \in \mathcal{Y}_\text{cand} \times \Theta_t} \langle \theta, \phi(x, y') - \phi(x, y_t^{(1)}) \rangle$ requires solving $\abs{\mathcal{Y}_\text{cand}}$ convex optimization problems: for each $y' \in \mathcal{Y}_\text{cand}$, solve $\widetilde{\theta}_t(y') \gets \argmax_{\theta \in \Theta_t} \langle \theta, \phi(x, y') - \phi(x, y_t^{(1)}) \rangle$, which has a linear objective and convex constraints. Then, select $y_t^{(2)} \gets \argmax_{y' \in \mathcal{Y}_\text{cand}} \langle \widetilde{\theta}_t(y'), \phi(x, y') - \phi(x, y_t^{(1)}) \rangle$ via another $\abs{\mathcal{Y}_\text{cand}}$ inner product evaluations.

Table~\ref{app-tab.walltime} reports average wall-time for a single step $t$ and its breakdown. As can be seen from these results, the wall-time in a single step is under a second, making it usable for real-world settings.

\begin{table}[h]
    \caption{Average wall-time for a step $t$ for \algoOurs{} across representative domains, broken down by the different computations done by the algorithm. Results are presented as mean values in seconds.}
    \vspace{1mm}
    \centering
    \scalebox{0.9}{
    \setlength\tabcolsep{4pt}
    \renewcommand{\arraystretch}{1.15}
    \begin{tabular}{lcccc}
        \toprule
        \textbf{Domain} & Wall-time Total (s) & Wall-time $\widehat{\theta}_t$ (s) & Wall-time $y_t^{(1)}$ (s) & Wall-time $y_t^{(2)}$ (s) \\
        \midrule
        \texttt{food64d} ($\abs{\mathcal{Y}_\text{cand}}{=}20$) & $0.1287$ & $0.0100$ & $0.0002$ & $0.1187$ \\
        \texttt{visual512d} ($\abs{\mathcal{Y}_\text{cand}}{=}40$) & $0.8719$ & $0.0328$ & $0.0003$ & $0.8392$ \\
        \bottomrule
    \end{tabular}}
    \label{app-tab.walltime}
\end{table}

\subsection{API and Embedding Costs}

Next, we give more details about the costs related to API calls and embedding. All costs are for a single run (steps capped at $49$) of \algoOurs{} with $\epsilon=0$, one question, and one GPT-based simulated user. 

\textbf{\texttt{food64d} domain.} Candidate pool generation with GPT-4o cost \$0.01; embedding all candidates was done locally with Potion \cite{minishlab2024model2vec} and took 0.01 seconds. The pairwise comparison stage with GPT-4o-mini cost \$0.01.

\textbf{\texttt{visual512d} domain.} Candidate pool generation with GPT-Image-1 cost \$1.71; embedding with OpenCLIP \cite{ilharco_gabriel_2021_5143773,datacomp} was done locally and took 12.13 seconds. The pairwise comparison stage cost \$0.12 using GPT-4o-mini.

\subsection{Notes About the Solver}

All optimization problems are solved with the \texttt{cvxpy} Python package using the \texttt{CLARABEL} conic solver with default settings. If \texttt{CLARABEL} is numerically unstable, the implementation falls back to \texttt{ECOS}, and then to \texttt{SCS} with \texttt{eps=1e-6} and a $50{,}000$ iteration cap.

    \clearpage
\section{Proofs}
\label{sec-app:proofs}

\begin{proof}[\textbf{Proof of \iftoggle{MainSuppContent}{Proposition~\ref{prop:conf}}{Proposition~1}}]
The result follows directly from Theorem~1 of~\cite{lee2024improved}.
\end{proof}

\begin{proof}[\textbf{Proof of \iftoggle{MainSuppContent}{Theorem~\ref{thm:pac}}{Theorem~1}}]
Let $\theta^\star \in \Theta_t$. If $\langle \theta^\star, \phi(x, y^\star) - \phi(x, y_\tau^{(1)}) \rangle > \epsilon$ holds, i.e., the returned response $y_\tau^{(1)}$ is worse than the best response $y^\star$ by $\epsilon$, then we have:
\begin{align*}
\langle \theta^\star, \phi(x, y^\star) - \phi(x, y_\tau^{(1)}) \rangle ~>~& \epsilon \\
~\geq~& \langle \widetilde{\theta}_\tau, \phi(x, y_\tau^{(2)}) - \phi(x, y_\tau^{(1)}) \rangle \\
~\geq~& \langle \theta^\star, \phi(x, y^\star) - \phi(x, y_\tau^{(1)}) \rangle , 
\end{align*}
where the second last inequality is due to the stopping condition of \iftoggle{MainSuppContent}{Algorithm~\ref{alg:general-pad}}{Algorithm~2}, and the last inequality is due to $(y_t^{(2)}, \widetilde{\theta}_t) \gets \argmax_{(y', \theta) \in \mathcal{Y}_\text{cand} \times \Theta_t} \langle \theta, \phi(x, y') - \phi(x, y_\tau^{(1)}) \rangle$ and $\theta^\star \in \Theta_t$. Note that $\theta^\star \in \Theta_t$ holds with probability at least $1-\delta$ according to \iftoggle{MainSuppContent}{Proposition~\ref{prop:conf}}{Proposition~1}.
\end{proof}

\begin{proof}[\textbf{Proof Sketch of \iftoggle{MainSuppContent}{Theorem~\ref{thm:stop-bound-2}}{Theorem~2}}]
The proof involves the following key steps:
\begin{itemize}
\item In Lemma~\ref{lem-theta-bound}, we obtain an upper bound on $\lVert \theta - \widehat{\theta}_t \rVert_{H_t(\theta^\star)}$.
\item In Lemma~\ref{lem-individual-bound}, we obtain an upper bound on $\lVert z_t \rVert_{H^{-1}_t(\theta^\star)}$, where $z_t = \phi(x, y_t^{(2)}) - \phi(x, y_t^{(1)})$.
\item In Lemma~\ref{lemma:stop-bound-2}, we use the above bounds to upper bound the number of times the pair of responses $(y^{(1)}, y^{(2)})$ is selected before the stopping time $\tau$; then, summing over all the response pairs provides an upper bound on the stopping time $\tau$.
\end{itemize}
\end{proof}

\begin{lemma}
Let $z_s = \phi(x, y_s^{(2)}) - \phi(x, y_s^{(1)})$ and $H_t(\theta^\star) = \sum_{s=1}^{t-1} \dot{\mu}(\langle \theta^\star, z_s \rangle) z_s z_s^\top + \lambda I_d$ with $\lambda = \frac{1}{4 S^2 (2 + 2S)}$. Then, for any $\theta \in \Theta_t$, the following holds with probability at least $1-\delta$:
\[
\lVert \theta - \widehat{\theta}_t \rVert_{H_t(\theta^\star)} ~\leq~ 2 S \sqrt{d \log \brr{e + \frac{S t}{d}} + \log \frac{1}{\delta}} .
\]
\label{lem-theta-bound}
\end{lemma}

\begin{proof}[\textbf{Proof of Lemma~\ref{lem-theta-bound}}]
Note that $\theta, \widehat{\theta}_t \in \Theta_t$. By using the triangle inequality, we have: 
\begin{align*}
\lVert \theta - \widehat{\theta}_t \rVert_{H_t(\theta^\star)} ~\leq~& \lVert \theta - \theta^* \rVert_{H_t(\theta^\star)} + \lVert \widehat{\theta}_t - \theta^* \rVert_{H_t(\theta^\star)} \\
~\leq~& S \sqrt{d \log \brr{e + \frac{S t}{d}} + \log \frac{1}{\delta}} + S \sqrt{d \log \brr{e + \frac{S t}{d}} + \log \frac{1}{\delta}} ,
\end{align*}
where the last inequality holds with probability at least $1-\delta$ due to Lemma~6 of \cite{lee2024improved}.
\end{proof}

\begin{lemma}
Let $z_s = \phi(x, y_s^{(2)}) - \phi(x, y_s^{(1)})$ and $H_t(\theta^\star) = \sum_{s=1}^{t-1} \dot{\mu}(\langle \theta^\star, z_s \rangle) z_s z_s^\top + \lambda I_d$ with $\lambda = \frac{1}{4 S^2 (2 + 2S)}$. Then, we have:
\[
\lVert z_t \rVert_{H^{-1}_t(\theta^\star)} ~\leq~ \sqrt{\frac{\kappa^\star_{\mathcal{X}, \mathcal{Y}}}{|\mathcal{E}_{y_t^{(1)}, y_t^{(2)}} (t-1)|}}
\]
\label{lem-individual-bound}
\end{lemma}

\begin{proof}[\textbf{Proof of Lemma~\ref{lem-individual-bound}}]
Let us define $\widetilde{z}_t = \sqrt{\dot{\mu}(\langle \theta^\star, z_t \rangle)} z_t$. Note that $H_t(\theta^\star) = \sum_{s=1}^{t-1} \widetilde{z}_s \widetilde{z}_s^\top + \lambda_t I_d$. Further, we have:
\[
\lVert z_t \rVert_{H^{-1}_t(\theta^\star)}^2 ~=~ \frac{1}{\dot{\mu}(\langle \theta^\star, z_t \rangle)} \lVert \widetilde{z}_t \rVert_{H^{-1}_t(\theta^\star)}^2 ~\leq~ \kappa^\star_{\mathcal{X}, \mathcal{Y}} \cdot \lVert \widetilde{z}_t \rVert_{H^{-1}_t(\theta^\star)}^2 .
\]
Let $z(y, y') = \phi(x, y') - \phi(x, y)$ and $\widetilde{z}(y, y') = \sqrt{\dot{\mu}(\langle \theta^\star, z (y, y') \rangle)} z (y, y')$. Then, note that $H_t(\theta^\star)$ can be written as follows:
\begin{align*}
H_t(\theta^\star) ~=~& \lambda I_d + \sum_{(y, y') \in \mathcal{Y}_\text{cand} \times \mathcal{Y}_\text{cand}} \abs{\mathcal{E}_{y, y'} (t-1)} \cdot \widetilde{z}(y, y') \widetilde{z}(y, y')^\top \\
~=~& A + B + C ,
\end{align*}
where
\begin{align*}
A ~=~& \lambda I_d \\
B ~=~& |\mathcal{E}_{y_t^{(1)}, y_t^{(2)}} (t-1)| \cdot \widetilde{z}(y_t^{(1)}, y_t^{(2)}) \widetilde{z}(y_t^{(1)}, y_t^{(2)})^\top \\
C ~=~& \sum_{(y, y') \in \mathcal{Y}_\text{cand} \times \mathcal{Y}_\text{cand} \backslash (y_t^{(1)}, y_t^{(2)})} \abs{\mathcal{E}_{y, y'} (t-1)} \cdot \widetilde{z}(y, y') \widetilde{z}(y, y')^\top
\end{align*}

Note that $\lambda I_d + z z^\top$ is a positive definite matrix in $\mathbb{R}^{d \times d}$ for any $z \in \mathbb{R}^d$ and $\lambda > 0$. Then, by repeatedly applying Sherman-Morrison formula~\cite{sherman1950adjustment} (see Lemma~3 of~\cite{xu2018fully}), we get:
\begin{align*}
& \lVert \widetilde{z}(y_t^{(1)}, y_t^{(2)}) \rVert_{H^{-1}_t(\theta^\star)}^2 \\
~=~& \widetilde{z}(y_t^{(1)}, y_t^{(2)})^\top H^{-1}_t(\theta^\star) \widetilde{z}(y_t^{(1)}, y_t^{(2)}) \\
~\leq~& \widetilde{z}(y_t^{(1)}, y_t^{(2)})^\top (A + B)^{-1} \widetilde{z}(y_t^{(1)}, y_t^{(2)}) \\
~=~& \widetilde{z}(y_t^{(1)}, y_t^{(2)})^\top \bss{\frac{I_d^{-1}}{\lambda_t} - \frac{|\mathcal{E}_{y_t^{(1)}, y_t^{(2)}} (t-1)| \frac{I_d^{-1}}{\lambda_t} \widetilde{z}(y_t^{(1)}, y_t^{(2)}) \widetilde{z}(y_t^{(1)}, y_t^{(2)})^\top \frac{I_d^{-1}}{\lambda_t}}{1 + |\mathcal{E}_{y_t^{(1)}, y_t^{(2)}} (t-1)| \widetilde{z}(y_t^{(1)}, y_t^{(2)})^\top \frac{I_d^{-1}}{\lambda_t} \widetilde{z}(y_t^{(1)}, y_t^{(2)})}} \widetilde{z}(y_t^{(1)}, y_t^{(2)}) \\
~=~& \frac{\lVert \widetilde{z}(y_t^{(1)}, y_t^{(2)}) \rVert^2_2}{\lambda_t} - \frac{\frac{|\mathcal{E}_{y_t^{(1)}, y_t^{(2)}} (t-1)|}{\lambda_t^2} \lVert \widetilde{z}(y_t^{(1)}, y_t^{(2)}) \rVert^4_2}{1 + \frac{|\mathcal{E}_{y_t^{(1)}, y_t^{(2)}} (t-1)|}{\lambda_t} \lVert \widetilde{z}(y_t^{(1)}, y_t^{(2)}) \rVert^2_2} \\
~=~& \frac{\frac{\lVert \widetilde{z}(y_t^{(1)}, y_t^{(2)}) \rVert^2_2}{\lambda_t}}{1 + \frac{|\mathcal{E}_{y_t^{(1)}, y_t^{(2)}} (t-1)|}{\lambda_t} \lVert \widetilde{z}(y_t^{(1)}, y_t^{(2)}) \rVert^2_2} \\
~\leq~& \frac{\frac{\lVert \widetilde{z}(y_t^{(1)}, y_t^{(2)}) \rVert^2_2}{\lambda_t}}{\frac{|\mathcal{E}_{y_t^{(1)}, y_t^{(2)}} (t-1)|}{\lambda_t} \lVert \widetilde{z}(y_t^{(1)}, y_t^{(2)}) \rVert^2_2} \\
~=~& \frac{1}{|\mathcal{E}_{y_t^{(1)}, y_t^{(2)}} (t-1)|} .
\end{align*}
\end{proof}

\begin{lemma}
Let $\tau$ be the stopping round of \algoOursLoss{}. Further, we define $\Delta_{\mathcal{Y}_\text{cand}} := \min_{y, y' \in \mathcal{Y}_\text{cand}; y \neq y'} \langle \theta^\star, \phi(x, y) - \phi(x, y') \rangle$. Then, with probability at least $1-\delta$, we have:
\[
\tau ~\leq~ \frac{4 S^2 K^2 \kappa^\star_{\mathcal{X}, \mathcal{Y}}}{\max\bcc{\epsilon, \Delta_{\mathcal{Y}_\text{cand}}}^2} \bss{d \log \brr{e + \frac{S \tau}{d}} + \log \frac{1}{\delta}} + K^2 .
\]
\label{lemma:stop-bound-2}
\end{lemma}

\begin{proof}[\textbf{Proof of Lemma~\ref{lemma:stop-bound-2}}]
Let $\tau$ be the stopping time of the algorithm. For any two response $y^{(1)}, y^{(2)} \in \mathcal{Y}_\text{cand}$, we define:
\[
\mathcal{E}_{y^{(1)}, y^{(2)}} (\tau) ~:=~ \bcc{\widetilde{t} \in [\tau]: \text{responses } y^{(1)} \text{ and } y^{(2)} \text{ are selected by \iftoggle{MainSuppContent}{Algorithm~\ref{alg:general-pad}}{Algorithm~2} for feedback}} .
\]
For every time step $\widetilde{t} \in \mathcal{E}_{y^{(1)}, y^{(2)}} (\tau)$, we have:
\begin{align}
y^{(1)} ~=~ y^{(1)}_{\widetilde{t}} ~=~& \argmax_{y \in \mathcal{Y}_\text{cand}} \langle \widehat{\theta}_{\widetilde{t}}, \phi(x, y) \rangle \label{eq:c1-select} \\
(y^{(2)}, \widetilde{\theta}_{\widetilde{t}}) ~=~ (y_{\widetilde{t}}^{(2)}, \widetilde{\theta}_{\widetilde{t}}) ~=~& \argmax_{(y', \theta) \in \mathcal{Y}_\text{cand} \times \Theta_{\widetilde{t}}} \langle \theta, \phi(x, y') - \phi(x, y^{(1)}) \rangle \label{eq:c2-select}
\end{align}
Note that $\widehat{\theta}_{\widetilde{t}}, \widetilde{\theta}_{\widetilde{t}} \in \Theta_{\widetilde{t}}$, and $\theta^\star \in \Theta_{\widetilde{t}}$ with probability at least $1-\delta$.

Since the stopping condition is violated in $\widetilde{t}$, we have:
\begin{equation}
\langle \widetilde{\theta}_{\widetilde{t}}, \phi(x, y^{(2)}) - \phi(x, y^{(1)}) \rangle ~\geq~ \epsilon .
\label{eq:combine-1a}
\end{equation}
Also, due to Eq.~\eqref{eq:c2-select}, we have:
\begin{align}
\langle \widetilde{\theta}_{\widetilde{t}}, \phi(x, y^{(2)}) - \phi(x, y^{(1)}) \rangle ~\geq~& \langle \theta^\star, \phi(x, y^{(2)}) - \phi(x, y^{(1)}) \rangle \nonumber \\
~\geq~& \min_{y, y' \in \mathcal{Y}_\text{cand}; y \neq y'} \langle \theta^\star, \phi(x, y) - \phi(x, y') \rangle \nonumber \\
~=~& \Delta_{\mathcal{Y}_\text{cand}} ,
\label{eq:combine-1b}
\end{align}
where the second inequality is due to the observation that $y^{(2)} \neq y^{(1)}$ (because of stopping condition violation). Then, by combining Eq.~\eqref{eq:combine-1a} and Eq.~\eqref{eq:combine-1b}, we get:
\begin{equation}
\langle \widetilde{\theta}_{\widetilde{t}}, \phi(x, y^{(2)}) - \phi(x, y^{(1)}) \rangle ~\geq~ \max\bcc{\epsilon, \Delta_{\mathcal{Y}_\text{cand}}} .
\label{eq:combine-1}
\end{equation}

Further, due to Eq.~\eqref{eq:c1-select}, we have:
\begin{equation}
\langle \widehat{\theta}_{\widetilde{t}}, \phi(x, y^{(1)}) - \phi(x, y^{(2)}) \rangle ~\geq~ 0 .
\label{eq:combine-2}
\end{equation}
Then, by combining Eq.~\eqref{eq:combine-1} and Eq.~\eqref{eq:combine-2}, we get:
\begin{align*}
\max\bcc{\epsilon, \Delta_{\mathcal{Y}_\text{cand}}} ~\leq~& \langle \widetilde{\theta}_{\widetilde{t}} - \widehat{\theta}_{\widetilde{t}}, \phi(x, y^{(2)}) - \phi(x, y^{(1)}) \rangle \\
~\leq~& \lVert \widetilde{\theta}_{\widetilde{t}} - \widehat{\theta}_{\widetilde{t}} \rVert_{H_{\widetilde{t}}(\theta^\star)} \cdot \lVert \phi(x, y^{(2)}) - \phi(x, y^{(1)}) \rVert_{H^{-1}_{\widetilde{t}}(\theta^\star)} \\
~\leq~& 2 S \sqrt{d \log \brr{e + \frac{S \widetilde{t}}{d}} + \log \frac{1}{\delta}} \cdot \lVert \phi(x, y^{(2)}) - \phi(x, y^{(1)}) \rVert_{H^{-1}_{\widetilde{t}}(\theta^\star)} ,
\end{align*}
where the last inequality holds with probability at least $1-\delta$ due to Lemma~\ref{lem-theta-bound}.

Now, let $\widetilde{t}$ be the largest value in $\mathcal{E}_{y^{(1)}, y^{(2)}} (\tau)$, i.e., the last time the pair of responses $(y^{(1)}, y^{(2)})$ selected before time step $\tau$. Then, for the time step $\widetilde{t}$, we have:
\begin{align*}
\max\bcc{\epsilon, \Delta_{\mathcal{Y}_\text{cand}}} ~\leq~& 2 S \sqrt{d \log \brr{e + \frac{S \widetilde{t}}{d}} + \log \frac{1}{\delta}} \cdot \lVert \phi(x, y^{(2)}) - \phi(x, y^{(1)}) \rVert_{H^{-1}_{\widetilde{t}}(\theta^\star)} \\
~\leq~& 2 S \sqrt{d \log \brr{e + \frac{S \widetilde{t}}{d}} + \log \frac{1}{\delta}} \cdot \sqrt{\frac{\kappa^\star_{\mathcal{X}, \mathcal{Y}}}{\abs{\mathcal{E}_{y^{(1)}, y^{(2)}} (\widetilde{t}-1)}}}
\end{align*}
where the first inequality holds with probability at least $1-\delta$ due to Lemma~\ref{lem-theta-bound} and the second inequality is due to Lemma~\ref{lem-individual-bound}. Hence, we can bound:
\begin{align*}
\abs{\mathcal{E}_{y^{(1)}, y^{(2)}} (\tau)} ~=~& \abs{\mathcal{E}_{y^{(1)}, y^{(2)}} (\widetilde{t}-1)} + 1 \\
~\leq~& \frac{4 S^2 \kappa^\star_{\mathcal{X}, \mathcal{Y}}}{\max\bcc{\epsilon, \Delta_{\mathcal{Y}_\text{cand}}}^2} \bss{d \log \brr{e + \frac{S \widetilde{t}}{d}} + \log \frac{1}{\delta}} + 1 \\
~\leq~& \frac{4 S^2 \kappa^\star_{\mathcal{X}, \mathcal{Y}}}{\max\bcc{\epsilon, \Delta_{\mathcal{Y}_\text{cand}}}^2} \bss{d \log \brr{e + \frac{S \tau}{d}} + \log \frac{1}{\delta}} + 1
\end{align*}
By summing over all the pairs of responses, we get:
\begin{align*}
\tau ~=~& \sum_{(y^{(1)}, y^{(2)}) \in \mathcal{Y}_\text{cand} \times \mathcal{Y}_\text{cand}} \abs{\mathcal{E}_{y^{(1)}, y^{(2)}} (\tau)} \\
~\leq~& \frac{4 S^2 K^2 \kappa^\star_{\mathcal{X}, \mathcal{Y}}}{\max\bcc{\epsilon, \Delta_{\mathcal{Y}_\text{cand}}}^2} \bss{d \log \brr{e + \frac{S \tau}{d}} + \log \frac{1}{\delta}} + K^2
\end{align*}
\end{proof}

\begin{proof}[\textbf{Full Proof of \iftoggle{MainSuppContent}{Theorem~\ref{thm:stop-bound-2}}{Theorem~2}}]
Let $\alpha = \log \brr{e + \frac{S \tau}{d}}$. Then, $\tau = \frac{d}{S} (e^\alpha - e)$. Thus, we can write the inequality in Lemma~\ref{lemma:stop-bound-2} as follows:
\begin{align*}
\frac{d}{S} (e^\alpha - e) ~\leq~& \frac{4 S^2 K^2 \kappa^\star_{\mathcal{X}, \mathcal{Y}}}{\max\bcc{\epsilon, \Delta_{\mathcal{Y}_\text{cand}}}^2} \bss{d \alpha + \log \frac{1}{\delta}} + K^2 , 
\end{align*}
which we can write as follows:
\[
e^\alpha ~\leq~ \frac{4 S^3 K^2 \kappa^\star_{\mathcal{X}, \mathcal{Y}}}{\max\bcc{\epsilon, \Delta_{\mathcal{Y}_\text{cand}}}^2} \cdot \alpha + \frac{4 S^3 K^2 \kappa^\star_{\mathcal{X}, \mathcal{Y}}}{\max\bcc{\epsilon, \Delta_{\mathcal{Y}_\text{cand}}}^2 d} \log \frac{1}{\delta} + \frac{K^2 S}{d}
\]
By letting $w = e^\alpha$, we have:
\[
w ~\leq~ A \log w + B
\]
where $A = \frac{4 S^3 K^2 \kappa^\star_{\mathcal{X}, \mathcal{Y}}}{\max\bcc{\epsilon, \Delta_{\mathcal{Y}_\text{cand}}}^2}$, and $B = \frac{4 S^3 K^2 \kappa^\star_{\mathcal{X}, \mathcal{Y}}}{\max\bcc{\epsilon, \Delta_{\mathcal{Y}_\text{cand}}}^2 d} \log \frac{1}{\delta} + \frac{K^2 S}{d}$.

Below, we consider the case $w > e$. If $w \leq e$, we already get the bound. For $w \geq e$, we have $\log w \geq 1$, so
\begin{align*}
w ~\leq~ A \log w + B ~\leq~ A \log w + B \log w ~=~ C \log w ,
\end{align*}
where $C := A + B$. Hence, we get:
\[
\frac{w}{\log w} ~\leq~ C .
\]
Define $f(t) = \frac{t}{\log t}$ for $t \geq e$. One can easily check that $f'(t) > 0$ for $t > e$. Thus, $f$ is strictly increasing on $(e, \infty)$. Therefore, the inequality $f(w) = \frac{w}{\log w} \leq C$ forces 
\[
w ~\leq~ t_0 , 
\]
where $t_0 > e$ is the unique solution of $f(t_0) = C$. Let $t_1 = 3 C \log C$. Then, for $C \geq 3$, we have:
\[
\log t_1 ~=~ \log 3 + \log C + \log \log D ~\leq~ \log C + \log C + \log C ~=~ 3 \log C .
\]
Hence, we get
\[
f(t_1) ~=~ \frac{t_1}{\log t_1} ~\geq~ \frac{3 C \log C}{3 \log C} ~=~ C ~=~ f(t_0) .
\]
Finally, due to monotonicity of $f(t)$ for $t > e$, we have:
\[
t_0 ~\leq~ t_1 ~=~ 3 C \log C , \text{ for all } C \geq 3 .
\]
Thus, we have: $w \leq \max \bcc{e, t_0} \leq \max \bcc{e, t_1}$. Therefore, $\tau = \frac{d}{S} (w - e) \leq \frac{d}{S} w \leq \frac{d}{S} \max \bcc{e, t_1} \leq \frac{d}{S} t_1$. Then, with probability at least $1-\delta$, we have:
\begin{align*}
\tau ~\leq~& \frac{3 d}{S} \cdot (A + B) \log (A + B) \\
~=~& 3 \cdot \bss{\frac{4 S^2 K^2 \kappa^\star_{\mathcal{X}, \mathcal{Y}}}{\max\bcc{\epsilon, \Delta_{\mathcal{Y}_\text{cand}}}^2} \brr{d + \log \frac{1}{\delta}} + K^2} \log \brr{\frac{4 S^3 K^2 \kappa^\star_{\mathcal{X}, \mathcal{Y}}}{\max\bcc{\epsilon, \Delta_{\mathcal{Y}_\text{cand}}}^2 d} \brr{d + \log \frac{1}{\delta}} + \frac{K^2 S}{d}} \\
~\leq~& 3 \cdot \bss{\frac{8 S^2 K^2 \kappa^\star_{\mathcal{X}, \mathcal{Y}}}{\max\bcc{\epsilon, \Delta_{\mathcal{Y}_\text{cand}}}^2} \brr{d + \log \frac{1}{\delta}}} \log \brr{\frac{8 S^3 K^2 \kappa^\star_{\mathcal{X}, \mathcal{Y}}}{\max\bcc{\epsilon, \Delta_{\mathcal{Y}_\text{cand}}}^2 d} \brr{d + \log \frac{1}{\delta}}} \\
~=~& 24 \cdot \Omega \cdot K^2 \cdot \log \brr{\frac{8 S}{d} \cdot \Omega \cdot K^2} \\
~=~& \mathcal{O}(\Omega K^2 \cdot \log (\Omega K^2)) ,
\end{align*}
where $\Omega : = \frac{S^2 \kappa^\star_{\mathcal{X}, \mathcal{Y}}}{\max\bcc{\epsilon, \Delta_{\mathcal{Y}_\text{cand}}}^2} \brr{d + \log \frac{1}{\delta}}$.

Since $\abs{\langle \theta, \phi(x, y) - \phi(x, y') \rangle} \leq \norm{\theta}_2 \cdot \norm{\phi(x, y) - \phi(x, y')}_2 \leq S$, we have $\epsilon \leq S$. Further note that $\kappa^\star_{\mathcal{X}, \mathcal{Y}} \geq 4$, $K \geq 1$, and $S \geq 1$ (we can scale up $S$). Thus, we have $C = A + B \geq A = \frac{4 S^3 K^2 \kappa^\star_{\mathcal{X}, \mathcal{Y}}}{\max\bcc{\epsilon, \Delta_{\mathcal{Y}_\text{cand}}}^2} \geq 16 S \geq 3$.
\end{proof}

}
 
}
{
}


\end{document}